\documentclass[11pt]{article} 
\usepackage{fullpage}
\usepackage{microtype}
\usepackage{graphicx}
\usepackage{natbib}
\usepackage{algorithm}
\usepackage{algorithmic}
\usepackage{color}
\usepackage{amsmath}
\usepackage{amsthm}
\usepackage{amssymb}
\usepackage{booktabs} 

\newcommand{\conf}[1]{}

\usepackage[colorlinks=true, allcolors=blue]{hyperref}
\hypersetup{
  colorlinks   = true,
  urlcolor     = blue,
  linkcolor    = blue,
  citecolor   = blue
}

\def\H{{\mathcal H}}

\newcommand{\ignore}[1]{}

\def\bold0{\mathbf{0}}

\newcommand{\cA}{\mathcal{A}}

\def\epsilon{\varepsilon}

\newtheorem{theorem}{Theorem}
\newtheorem{lemma}{Lemma}

\newtheorem{definition}{Definition}

\newcommand{\cO}{\mathcal{O}}
\newcommand{\cP}{\mathcal{P}}
\newcommand{\cQ}{\mathcal{Q}}
\newcommand{\cX}{\mathcal{X}}
\newcommand{\cY}{\mathcal{Y}}

\newcommand{\cV}{\mathcal{V}}
\newcommand{\cW}{\mathcal{W}}
\newcommand{\cH}{\mathcal{H}}
\newcommand{\cU}{\mathcal{U}}

\newcommand{\mult}{\text{Mul}}

\newcommand{\namedref}[2]{\mbox{\hyperref[#2]{#1~\ref*{#2}}}}

\newcommand{\figurerefb}[2]{\mbox{\hyperref[#1]{Figure~\ref*{#1}#2}}}

\newcommand{\equationref}[1]{\mbox{\hyperref[#1]{(\ref*{#1})}}}
\renewcommand{\eqref}{\equationref}

\newcommand{\EE}{\mathbb{E}}

\newcommand{\RR}{\mathbb{R}}

\newcommand{\indic}{1}
\newcommand{\Bin}{\text{Bin}}

\usepackage{subcaption}
\begin{document}

\title{Learning discrete distributions: 
user vs item-level privacy}
\author{
\begin{tabular}[t]
{c@{\extracolsep{6.5em}}c}
  Yuhan Liu & Ananda Theertha Suresh\\
 Cornell University & Google Research\\ 
\small \texttt{yl2976@cornell.edu} & \small \texttt{theertha@google.com}
\end{tabular}
\vspace{2ex}\\
\begin{tabular}[t]{c@{\extracolsep{6.5em}}c@{\extracolsep{6.5em}}c}
  Felix Yu & Sanjiv Kumar & Michael Riley\\
  Google Research & Google Research & Google Research \\
 \small \texttt{felixyu@google.com} & \small\texttt{sanjivk@google.com}
 & \small \texttt{riley@google.com}
\end{tabular}
}

\maketitle

\begin{abstract}
Much of the literature on differential privacy
focuses on item-level privacy, where
loosely speaking, the goal is to provide
privacy per item or training example.
However, recently many practical applications such as federated learning require preserving privacy for all items of a single user, which is much harder to achieve. Therefore understanding the theoretical limit of user-level privacy becomes crucial. 

We study the fundamental problem of learning discrete distributions over $k$ symbols with user-level differential privacy. If each user has $m$ samples, we show that 
straightforward applications of Laplace or Gaussian mechanisms require the number of users to be 
$\mathcal{O}(k/(m\alpha^2) + k/\epsilon\alpha)$ to achieve an $\ell_1$ distance of $\alpha$ between the true and estimated distributions, with the privacy-induced penalty $k/\epsilon\alpha$ independent of the number of samples per user $m$.
Moreover, we show that any mechanism that only operates on the final aggregate counts should require a user complexity of the same order.
We then propose a mechanism such that the number of users scales as  $\tilde{\mathcal{O}}(k/(m\alpha^2) + k/\sqrt{m}\epsilon\alpha)$ 
and hence the privacy penalty is $\tilde{\Theta}(\sqrt{m})$ times smaller compared to the standard mechanisms
in certain settings of interest. We further show that the proposed mechanism is nearly-optimal under certain regimes.

We also propose general techniques for obtaining lower bounds on restricted differentially private estimators and a lower bound on the total variation between binomial distributions, both of which might be of independent interest.
\end{abstract}

\section{Introduction}

\subsection{Differential privacy}
Differential privacy (DP)~\citep{CynthiaFKA06, DworkR14, WassermanZ10} has emerged as the standard framework for providing privacy for various statistical problems. Ever since its inception, it has been applied to various statistical and learning scenarios including learning histograms~\citep{CynthiaFKA06, hay2010boosting, suresh2019differentially}, statistical estimation~\citep{DiakonikolasHS15, KamathLSU18, acharya2020differentially, KamathSU20, AcharyaCT19, AcharyaSZ18a}, learning machine learning models~\citep{chaudhuri2011differentially, bassily2014private, mcmahan2017learning, DworkTTZ14}, hypothesis testing~\citep{AliakbarpourDR18, AcharyaSZ18}, and various other tasks. 

Differential privacy is studied in two scenarios, local differential privacy~\citep{kasiviswanathan2011can, duchi2013local} and global differential privacy~\citep{CynthiaFKA06}. In this paper, we study the problem under the lens of global differential privacy, where the goal is to protect the privacy of the algorithm outcomes.  Before we proceed further, we first define differential privacy.

\begin{definition}
A randomized mechanism $\mathcal{M}: \mathcal{D} \rightarrow \mathcal{R}$ with domain $\mathcal{D}$
 and range $\mathcal{R}$ satisfies $(\epsilon, \delta)$-differential privacy if for any two adjacent datasets 
 $D, D' \in \mathcal{D}$ and for any subset of output 
 $\mathcal{S} \subseteq \mathcal{R}$, it holds that 
 \[
 \Pr[\mathcal{M}(D) \in \mathcal{S}] \leq e^{\epsilon} \Pr[\mathcal{M}(D') \in \mathcal{S}] + \delta.
 \]
 \end{definition}
 If $\delta=0$, then the privacy is also referred to as pure differential privacy.
 
An important aspect of the above definition is the notion of neighboring or adjacent datasets. If a dataset $D$ 
is a collection of $n$ items $x_1, x_2, \ldots, x_n$, then typically adjacent datasets are defined as those that differ in a single  item $x_i$~\citep{CynthiaFKA06}. 
 
However, in practice, each user may have many items and may wish to preserve privacy for all of them. Hence, this simple definition of item-level neighboring datasets would not be enough. For example, if each user has infinitely many points of the same example, then the bounds become vacuous. 
 
Motivated by this, user-level privacy was proposed recently. 
Formally, given $s$ users where each user $u$ has
$m_u$ items $x_1(u), x_2(u), \ldots x_{m_u}(u)$, then two datasets are adjacent if they differ in data of a single user. For example, in the simple setting when each user has $m$ samples, if two datasets are adjacent in user-level privacy, they could differ in at most $m$ items under the definition of item-level privacy.
 
Since user-level privacy is more practical, it has been studied in the context of 
learning machine learning models via federated learning~\citep{mcmahan2017learning, mcmahan2018general,wang2019beyond,  augenstein2019generative}. The problem of bounding user contributions in user-level privacy in the context of both histogram estimation and learning machine learning models was studied in~\cite{amin2019bounding}. Differentially private SQL with bounded user contributions was proposed in~\cite{wilson2019differentially}. Understanding trade-offs between utility and privacy in the context of user-level global DP is one of the challenges in federated learning~\citep[Section 4.3.2]{kairouz2019advances}. \cite{kasiviswanathan2013analyzing} studied node differential privacy which guarantees privacy in the event of adding or removing nodes in network data.

Our goal is to understand theoretically the utility-privacy trade-off for user-level privacy and
compare it to the item-level counterpart. To this end, we study the problem of learning discrete distributions under user and item-level privacy.

\subsection{Learning discrete distributions}

Learning discrete distributions is a fundamental problem in statistics
with practical applications that include language modeling, ecology, and databases. In many applications, the underlying data distribution is private and sensitive  e.g., learning a language model from user-typed texts. To this end, learning discrete distributions under differential privacy has been studied extensively with various loss functions and non-asymptotic convergence rates~\citep{braess2004bernstein, kamath2015learning, han2015minimax}, with local differential privacy
\citep{duchi2013local,kairouz2016discrete,AcharyaCT19, YeB18}, with global differential privacy~\citep{DiakonikolasHS15, acharya2020differentially}, and with communication constraints~\citep{barnes2019lower, AcharyaCT19}, among others.

Before we proceed further, we first describe the learning scenario.
Let $p$ be an unknown distribution over
symbols $1,2,\ldots, k$ i.e.,
$\sum_i p_i = 1$ and $p_i \geq 0$ for all $i \leq k$. Let $\Delta_k$ be the set of all discrete distributions over the domain  $[k]:=\{1,2,\ldots, k\}$. 

Suppose there are $s$ users indexed by $u$, and let $\cU$ denote the set of all users. We assume that each user $u$ has $m$ i.i.d. samples $X^m(u) = [X_1(u), X_2(u),\ldots, X_m(u)]\in\cX:=[k]^m$ from the same distribution $p$. 

We extend our results to the case when users have different number of samples in Appendix~\ref{app:extensions}. However, we assume that all users have samples from the same distribution throughout the paper. Extending the algorithms to scenarios where users have samples from different distributions is an interesting open direction.

Let $X^s = [(u, X^m(u)) : u \in \cU]$ be the set of user and sample pairs. Let $\cX^s$ be the collection of all possible user-sample pairs. For an algorithm $A$, let  $\hat{p}^A(X^s)$ be its output, a mapping from $\mathcal{X}^{s} \mapsto \Delta_k$. The performance  for a given sample $X^s$ is measured in terms of $\ell_1$ distance, $
\ell_1(p, \hat{p}^A) = \sum_{i=1}^k |p_i - \hat{p}_{i}^A(X^s)|
$.  We measure the performance of the estimator 
for a distribution $p$ 
by its expectation over the algorithm and samples i.e.,
$
L(A, s, m, p)  = \EE_{A, X^s} [\ell_1(p, \hat{p}^A(X^s))]$.

We define the user complexity of an algorithm $A$ as the minimum number of users required to achieve error at most $\alpha$ for all distributions:
\begin{equation}
\label{eq:sam_A}
S^A_{m,\alpha} = \min_{s} \{s: \sup_{p \in \Delta_k} L(A, s, m, p) \leq \alpha \}.
\end{equation}
The min-max user complexity is
\[
S^*_{m,\alpha} = \min_{A} S^A_{m,\alpha}. 
\]

Well known results on non-private discrete distribution estimation (see~\citep{kamath2015learning, han2015minimax}) characterize the min-max user complexity as
\begin{equation}
    S^*_{m,\alpha} = \Theta\left(\frac{k}{m \alpha^2} \right).
    \label{equ:non-private-lb}
\end{equation}

Let $\cA_{\epsilon, \delta}$ be the set of all $(\epsilon, \delta)$ differentially private algorithms. Similar to~\eqref{eq:sam_A}, for a differentially private algorithm $A$, let 
$S^A_{m,\alpha, \epsilon, \delta}$ be the minimum of samples necessary to achieve $\alpha$ error for all distributions $p \in \Delta_p$ with $(\epsilon, \delta)$ differential privacy. We are interested in characterizing and developing polynomial-time algorithms that achieve the min-max user complexity of $(\epsilon, \delta)$ differentially private mechanisms.
\[
S^*_{m,\alpha, \epsilon, \delta} =
 \min_{A \in \cA_{\epsilon, \delta}} S^A_{m,\alpha, \epsilon, \delta}.
\]

\section{Previous results}

The min-max rate of learning discrete distributions for item-level privacy, which corresponds to $m=1$, was studied by~\cite{DiakonikolasHS15} and~\cite{acharya2020differentially}. They showed that for any $(\epsilon, \delta)$ estimator,
\[
S^*_{1,\alpha, \epsilon, \delta} =
\Theta \left( \frac{k}{\alpha^2} + \frac{k}{\alpha (\epsilon+\delta)} \right).
\]
The goal of our work is to understand the behavior of $S^*_{m,\alpha, \epsilon, \delta}$ w.r.t. m.
We first discuss a few natural algorithms and analyze their user complexities.

One natural algorithm is for each user to sample one item and use known results from item-level privacy. Such a result would yield,
\[
S^{\text{sample}}_{m,\alpha, \epsilon, \delta} =
\cO \left( \frac{k}{\alpha^2} + \frac{k}{\alpha (\epsilon+\delta)} \right).
\]

The other popular algorithms are Laplace or Gaussian mechanisms that rely on counts of users. For a particular user sample $X^m(u)$, let $N(u) = [N_1(u), \dots, N_k(u)]$, be the vector of counts. 
A natural algorithm is to sum all the user contributions to obtain the overall count vector $N$, where the count of a symbol $i$ is given by
\[
N_i = \sum_u N_i(u).
\]
Finally a non- private estimator can be obtained by computing the empirical estimate:
\[
\hat{p}^{\text{emp}}_i = \frac{N_i}{ms}.
\]
To obtain a differentially private version of the empirical estimate, one can add Laplace or Gaussian noise with some suitable magnitude. To this end, we need to compute the sensitivity of the empirical estimate.

Recall that two datasets $D, D'$ are adjacent if there exists a single user $u$ such that $N(u, D) \ne N(u, D')$, and $N(v, D) = N(v, D')$ for all $v \in \mathcal{U}$ and $v \ne u$. Therefore the $\ell_1$ sensitivity is
\[
\Delta_1(N) = \max_{D, D' \textnormal{ adjacent}}||N(D) - N(D')||_1 = 2m.
\]
and the $\ell_2$ sensitivity is
\[
\Delta_2(N) = \max_{D, D' \textnormal{ adjacent}}||N(D) - N(D')||_2 = \sqrt{2}m.
\]

A widely used method is the Laplace mechanism, which ensures $(\epsilon, 0)$ differential privacy. \begin{definition}
Given any function $f$ that maps the dataset to $\RR^k$, let the $\ell_1$ sensitivity $\Delta(f) = \max_{D, D'  \textnormal{ adjacent} } || f(D) - f(D') ||_1$. The Laplace mechanism is defined as
\[
\mathcal{M}(D, f(\cdot), \epsilon) = f(D) + (Y_1, \dots, Y_k),
\]
where $Y_i$ are i.i.d random variables drawn from Lap$(\Delta f / \epsilon)$.
\end{definition}
The Gaussian mechanism is defined similarly with $\ell_2$ sensitivity and 
Gaussian noise. We first analyze Laplace and Gaussian mechanisms under user-level privacy.
\begin{lemma}
\label{lem:naive}
For the Laplace mechanism, given by
$
\hat{p}^{\text{l}}_i = \hat{p}^{\text{emp}}_i + \frac{Z_i}{ms},
$
where $Z_i = Lap(2m/\epsilon)$, 
\[
S^l_{m,\alpha, \epsilon, 0}= 
\cO \left(\frac{k}{m \alpha^2} +\frac{k}{\alpha \epsilon} \right).
\]
Similarly if $\epsilon \leq 1$, for the Gaussian mechanism, given by
$
\hat{p}^{\text{g}}_i = \hat{p}^{\text{emp}}_i + \frac{Z_i}{ms},
$
where $Z_i = \mathcal{N}(0, 4\log(1.25/\delta)m^2/\epsilon^2)$, 
\[
S^g_{m,\alpha, \epsilon, \delta}
=
\cO \left(\frac{k}{m \alpha^2} +\frac{k}{\alpha \epsilon} \sqrt{\log \frac{1}{\delta}} \right).
\]
\end{lemma}
The proof follows from the definitions of the Laplace and Gaussian mechanisms, which we provide in Appendix~\ref{app:laplace} for completeness. The non-private user complexity term $\cO(k/(m\alpha^2))$ decreases with the number of samples from user $m$, but somewhat surprisingly the additional term due to privacy $\cO(k/\alpha\epsilon)$ is independent of $m$. In other words, no matter how many samples each user has, it does not help to reduce the privacy penalty in the user complexity. This could be especially troublesome when $m$ gets large, in which case the privacy term dominates the user complexity. 

\section{New results}

We first ask if the above results on Laplace and Gaussian mechanisms are tight. We show that they are by proving a lower bound on a wide class of estimators that only rely on the final count. The proof is based on a new coupling technique with details explained in Section \ref{sec:lower-bounds} . 

\begin{theorem}
Let $\epsilon +\delta< c$, where $c$ is determined in the proof later.
Let $A$ be any $(\epsilon, \delta)$ mechanism
that only operates on summed counts of all users $N = [N_1, N_2, \ldots N_k]$ directly. Then,
\[
S^A_{m, \alpha, \epsilon, \delta}
= 
\Omega\left( \frac{k}{m\alpha^2} + \frac{k}{\alpha(\epsilon +\delta)}\right).
\]
\label{thm:lb-restricted-Assouad}
\end{theorem}
The above lower bound suggests that any algorithm that only operates on the final count aggregate would incur additional cost for user complexity independent of $m$ due to privacy restriction. However it may not apply to algorithms that do not solely rely on the counts, which justifies the need to design algorithms beyond straightforward applications of the Laplace or Gaussian mechanisms.

We proceed to design algorithms that exceed the above user-complexity limit. The first one is for the dense regime where $k \leq m$: on average each user sees most of the high-probability symbols. The second one is for the sparse regime where $k \geq m$: users don't see many symbols. By combining the two of them, we get the following improved upper bound on min-max user complexity.
\begin{theorem}
\label{thm:kall}
Let $\epsilon \leq 1$. There exists a polynomial time algorithm $(\epsilon, \delta)$-differentially private algorithm $A$ such that
 \begin{equation}
 \label{eq:kall_main}
S^A_{m, \alpha, \epsilon, \delta} =
\cO \left(
\log \frac{km}{\alpha} \cdot 
\max \left( 
\frac{k}{m\alpha^2} + \frac{k}{\sqrt{m}\alpha\epsilon} \sqrt{\log \frac{1}{\delta}}, \frac{\sqrt{k}}{\epsilon}  \sqrt{\log \frac{1}{\delta}} \right) \right).
\end{equation}
\end{theorem}
The algorithm in Theorem~\ref{thm:kall} assumes that all users have the same number of samples. When $k$ is large or $\alpha$ is small, the first term in the maximum dominates and  we obtain $\tilde{\Theta}(\sqrt{m})$ improvement compared to Laplace and Gaussian mechanisms. In Appendix~\ref{app:extensions}, we modify it to the setting when users have different number of samples. The sample complexity is similar to~\eqref{eq:kall_main}, with $m$ replaced by $\bar{m}$, the median of number of samples per user. We also note that our algorithms are designed using high probability arguments, and hence we can easily obtain the sample complexity with logarithmic dependence on the inverse of the confidence parameter. 

Finally we provide an information theoretic lower bound for any $(\epsilon,0)$-differentially private algorithm:
\begin{theorem}
\label{thm:lowerall}
Let $\epsilon \leq 1$. Then
 \[
S^*_{m, \alpha, \epsilon, 0} =
\Omega \left(
\frac{k}{m\alpha^2} + \frac{k}{\sqrt{m}\alpha\epsilon} \right).
\]
\end{theorem}
Theorems~\ref{thm:kall} and~\ref{thm:lowerall} resolve the user complexity of learning discrete distributions up to log factors and the $\delta$-term in privacy. It would be interesting to see if Theorem~\ref{thm:lowerall} can be extended to nonzero values of $\delta$. In the next two sections, we first analyze the lower bounds and then propose algorithms. 

\section{Lower bounds}
\label{sec:lower-bounds}
The $\Omega(k/(m\alpha^2))$ part of the user-complexity lower bounds in Theorem \ref{thm:lb-restricted-Assouad} and \ref{thm:lowerall} follows from classic non-private results \eqref{equ:non-private-lb}. Therefore in this section we focus on the private part.

\subsection{Lower bound for restricted estimators}
We first start with the lower bound for algorithms that work directly on the counts vector $N = [N_1, N_2, \ldots, N_k]$, even though the learner has access to $\{N(u) : u \in \cU\}$. This motivates the definition of restricted estimators, which only depends on some function of the observation rather than the observation itself. 

\begin{definition}[$f$-restricted estimators]
Let $f:\cX^s\mapsto \cY$ which maps users' data to some domain $\cY$. An estimator $\hat{\theta}$ is $f$-restricted if it has the form $\hat{\theta}(X^s)=\hat{\theta}'(f(X^s))$ for some function $\hat{\theta}'$.
\end{definition}

We generalize 
Assouad's lemma~\citep{Assouad83,yu1997assouad} with differential privacy and the restricted estimators using the recent coupling methods of~\cite{AcharyaSZ18, acharya2020differentially}. These bounds could be of interest in other applications and we describe a general framework where they are applicable. 

Let $\cX$ be some domain of interest and $\cP$ be any set of distributions over $\cX$.

Assume that $\mathcal{P}$ is parameterized by $\theta:\mathcal{P}\mapsto\Theta\in \RR^d$, i.e. each $p\in\mathcal{P}$ can be uniquely represented by a parameter vector $\theta(p)\in\RR^d$. Given $s$ samples from an unknown distribution $p \in \cP$, an estimator $\hat{\theta}:\cX^s\mapsto \Theta$ takes in a sample from $\cX^s$ and outputs an estimation in $\Theta$. Let $\ell:\Theta\times\Theta\mapsto \RR_+$ be a pseudo-metric that measures estimation accuracy. For a fixed function $f$, let $\cA_f$ be the class of $f$-restricted estimators.
We are interested in the min-max risk for $(\epsilon, \delta)$-DP restricted estimators:
\[
L(\cP, \ell, \epsilon, \delta):=\min_{\hat{\theta}\in\cA_{\epsilon, \delta} \cap \cA_f}\max_{p\in\mathcal{P}}\EE_{X^s\sim p^s}[\ell(\hat{\theta}(X^s), \theta(p))].
\]

We need two more definitions to state our results.
\begin{definition}[$f$-identical in distribution]
Given a function $f$, two random variables $X$ and $Y$ are $f$-identical in distribution if $f(X)$ and $f(Y)$ have the same distributions, denoted by $Y\sim_f X$.  If $X\sim p$ and $Y\sim p'$ , then we can also say $p\sim_f p'$. 
\end{definition}

\begin{definition}[$f$-coupling]
Given a function $f$ and two distributions $p, q$, random variables $(X, Y)$ are an $f$-coupling of $p$ and $q$ if $X\sim_f p$ and $Y\sim_f q$. When $f$ is the identity mapping, then an $f$-coupling is same as standard coupling.
\end{definition}

We make the following observation for restricted estimators: since we can only estimate the true parameter $\theta$ through some function $f$ of the observation $X^s$, then any random variable $Y^s$ such that $f(Y^s)$ has the same distribution as $f(X^s)$ would yield the same distribution for restricted estimators $\hat{\theta}$. Thus, if $\hat{\theta}$ could distinguish two distributions $p_1, p_2$ from the space of product distributions $\cP^s:=\{p^s: p\in\cP\}$, then it should also be able to distinguish $p_1'\sim_f p_1$ and $ p_2'\sim_f p_2$. We are able to prove tighter lower bounds because $p_1', p_2'$ (potentially outside of $\mathcal{P}^s$) could be harder to distinguish than the original distributions $p_1, p_2$. This is the most significant difference between our method and \citep{acharya2020differentially}, whose argument does not capture the above observation for restricted estimators and hence requires designing testing problems within the original class of distributions.

With this intuition, we show a generalization of Assouad's lower bound in Theorem \ref{thm:Assouad}. It relies on an extension of the Le Cam's method~\citep{LeCam73, yu1997assouad}. The proofs are in Appendix \ref{app:assouad}. For two sequences $X^s$ and $Y^s$, let
$d_{h}(X^s, Y^s) = \sum^s_{i=1} 1_{X_i \neq Y_i}$ denote the Hamming distance.
\begin{theorem}[$(\epsilon, \delta)$-DP Assouad's method for restricted estimators] Let $\mathcal{V}:=\{\pm 1\}^k$ be a hypercube. Consider a set of distributions $\mathcal{P}_{\mathcal{V}}:=\{p_\nu: \nu\in \cV\}$ over $\cX$. Let for all $u, v\in\mathcal{V} $ the loss $\ell$ satisfies
\begin{equation}
\label{eq:tau}
  \ell(\theta(p_u), \theta(p_v))\ge 2\tau \sum_{i=1}^k\indic[u_i\ne v_i].
\end{equation}
For each $i\in[k]$, define the following mixture of product distributions:
\[
p_{+i}^s=\frac{2}{|\mathcal{V}|}\sum_{v\in\mathcal{V}:v_i=+1}p^s_v,\quad p_{-i}^s=\frac{2}{|\mathcal{V}|}\sum_{v\in\mathcal{V}:v_i=-1}p^s_v.
\]
If for all $ i\in[k]$ there exists an $f$-coupling $(X^s, Y^s)$ between $p^s_{+i}$ and $p^s_{-i}$ with $\EE[d_{h}(X^s, Y^s)]\le D$, then for any restricted estimator $\hat{\theta} \in \cA_f \cap \cA_{\epsilon, \delta}$,
\[
\sup_{p\in\mathcal{P}}\EE_{X^s\sim p^s}\ell(\theta(p), \hat{\theta}(X^s)) \ge \max\left(\frac{\tau}{2}\sum_{i=1}^k(1-d_{TV}(p_{+i}^s, p_{-i}^s)),  \frac{k\tau}{2}\left(0.9e^{-10\varepsilon D}-10D\delta\right)\right).
\]
\label{thm:Assouad}
\end{theorem}

The proof of Theorem~\ref{thm:lb-restricted-Assouad} follows from Theorem~\ref{thm:Assouad}. We provide details in Appendix \ref{app:thm1_detail}. 
\begin{proof}[Proof sketch of Theorem \ref{thm:lb-restricted-Assouad}]
In our problem setting, $\cX=[k]^m$ is the domain and $\cP$ is the set of multinomial distributions $\cP = \{\mult(m, p) : p \in \Delta_k\}$, where $\mult(m,p)$ denotes the multinomial distribution.
The parameter we are trying to estimate 
is the underlying $p$ and the loss is $\ell_1$ distance. 

We construct $\mathcal{P}_{\mathcal{V}}$ as follows: let $\alpha\in(0, 1/6)$, and for each $\nu\in \mathcal{V}:=\{-1, 1\}^{k/2}$,

\begin{equation}
    p_\nu =  \mult \left(m , \frac{1}{k}(1+3\alpha \nu_1, 1-3\alpha\nu_1, ..., 1+3\alpha\nu_{k/2}, 1-3\alpha\nu_{k/2}) \right).
    \label{equ:paninsky}
\end{equation}
For any $u, v\in\cV$, 
$\ell_1$ distance satisfies~\eqref{eq:tau} with $\tau = 6\alpha/k$. 

For restricted estimator $\hat{p}^A$ which only operates on  $N=[N_1, ..., N_k]$, for each $i\in[k]$ we can design an $N$-coupling $(X^s, Y^s)$ of $p_{+i}^{s}$ and $p_{-i}^{s}$ with $\EE[d_{h}(X^s, Y^s)]\le6\alpha s/k + 1 =: D$.  Plugging in $\tau$ and $D$ in Theorem~\ref{thm:Assouad} yields the desired min-max rate and user complexity.
\end{proof}

\subsection{Lower bound for the general case}
We provide the complete proof of Theorem \ref{thm:lowerall} in Appendix~\ref{app:fano} and sketch an outline here. We use differentially private Fano's method ~\cite[Corollary 4]{acharya2020differentially}.
We  design a set of distributions $\cP\subseteq\Delta_k$ such that,
$|\cP|=\Omega(\exp(k))$, and for each $p, q\in\cP$,
\[
\ell_1(p, q)=\Omega (\alpha),\; d_{KL}(\mult(p)||\mult(q))=O(m\alpha^2),\; d_{TV}(\mult(p)||\mult(q))=O(\sqrt{m}\alpha^2).
\]
Applying \citet[Corollary 4]{acharya2020differentially} with $M=\Omega(\exp(k)), \tau = \alpha, \beta = O(m\alpha^2),\gamma = O(\sqrt{m}\alpha^2)$ yields the result.

\section{Algorithms}
We first propose an algorithm for the dense regime where $k \leq m$. In this regime, on average each user sees most of the high-probability symbols. However, this algorithm does not extend directly to the sparse regime when $k \geq m$. In the sparse regime, we augment the dense algorithm regime with another sub-routine for small probabilities. Both algorithms could be extended to the case when users have different number of samples (see Appendix \ref{app:extensions}). 

\subsection{Algorithms for the dense regime}

We first motivate our algorithm with an example. Consider a symbol with probability around $1/2$. If $m$ is large, then by the Chernoff bound, such a symbol has counts in the range \[
\left[\frac{m}{2} - \sqrt{\frac{m}{2}\log \frac{2}{\delta}},\frac{m}{2} + \sqrt{\frac{m}{2} \log \frac{2}{\delta}} \right],
\]
with probability $\geq 1- \delta$. Hence, neighboring datasets differ typically with $\sqrt{m}$ counts. However, in the worst case, they could differ by $m$ and hence standard mechanisms add noise proportional to $m$.

We propose the following alternative method. The count for symbol $i\in[k]$ can take values from $0,1,\ldots, m$ and is distributed according to $\Bin(m,p_i)$. Thus, we can learn this distribution $\Bin(m,p_i)$ itself to a good accuracy and then estimate $p_i$ from the estimated density of $\Bin(m,p_i)$.

We propose to use the private hypothesis selection algorithm due to~\cite{bun2019private} to learn the density of the Binomial distribution. It gives a score for every hypothesis using the Scheff\'e algorithm~\citep{scheffe1947useful} and then privately selects a hypothesis using the exponential mechanism based on the score functions. For completeness, we state the private hypothesis selection algorithm in Algorithm~\ref{alg:phs} and its guarantee in Lemma \ref{lem:phs} in the Appendix.

\begin{algorithm}[t]
\caption{Private hypothesis selection: PHS($\cH, D, \alpha, \epsilon$) \citep{bun2019private}}
\label{alg:phs}
\begin{algorithmic}[1]
\STATE \textbf{Input}: $\cH=\{H_1, \ldots, H_d\}$ the set of hypotheses, dataset $D$ of $s$ samples drawn i.i.d. from $p\in\cH$, accuracy parameter $\alpha\in(0, 1)$, privacy parameter $\epsilon$.
\FOR {\textbf{each} $H_i, H_j\in \cH$}
\STATE $\cW=\{x\in\cX: H_i(x)>H_j(x)\}$, $p_i=H_i(\cW), p_j=H_j(\cW)$.
\STATE Compute $\hat{\tau}=\frac{1}{s}|\{x\in D: x\in \cW\}|$ and
\[\Gamma(H_i, H_j, D)=
\begin{cases}
s & p_i-p_j\le 3\alpha;\\
s\cdot \max\{0, \hat{\tau}-(p_j+(3/2)\alpha)\}&\text{otherwise}.
\end{cases}
\]
\ENDFOR
\STATE For each $H_j\in\cH$ compute $S(H_j, D)=\min_{H_k\in\cH}\Gamma(H_j, H_k, D)$.
\RETURN random hypothesis $\hat{H}$ such that for each $H_j$:
$$
\Pr[\hat{H}=H_j]\propto \exp\left(\frac{S(H_j, D)}{2\epsilon}\right).
$$
\end{algorithmic}
\end{algorithm}

\begin{algorithm}[t]
\caption{Learning binomial distributions: Binom($D, \epsilon, \alpha$)}
\label{alg:binom}
\begin{algorithmic}[1]
\STATE \textbf{Input}: Dataset $D$ of $s$ samples i.i.d. from $\Bin(m, p)$, privacy parameter $\epsilon$, accuracy parameter $\alpha$.
\STATE Let $\cP=\{0,\frac{c\alpha}{20m},\frac{2c\alpha}{20m},\ldots, 1\rceil \}$ and $\cH=\{\Bin(m, p): p\in\cP\}$. 
\STATE Run PHS($\cH, D, c\alpha/5, \epsilon$) and obtain $\Bin(m, \hat{p})$.
\RETURN $\hat{p}$.
\end{algorithmic}
\end{algorithm}

\begin{figure}[t]
    \centering
    \begin{subfigure}[t]{.43\linewidth}
        \centering
        \includegraphics[height=4cm]{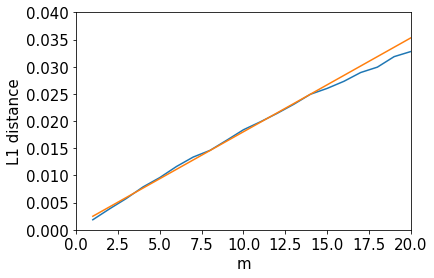}
        \caption{$m \leq 20$. $\ell_1$ distance grows linearly with $m$.}
        \label{fig:small_m}
    \end{subfigure}
    \begin{subfigure}[t]{.45\linewidth}
        \centering
        \includegraphics[height=4cm]{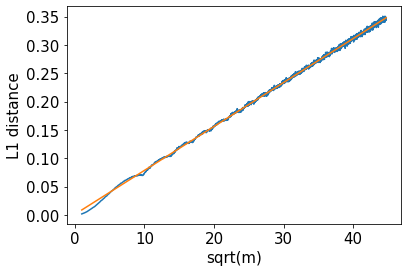}
        \caption{Larger values of $m$. $\ell_1$ distance grows as $\sqrt{m}$.}
        \label{fig:large_m}
    \end{subfigure}
    \caption{$\ell_1(\Bin(m, p),\Bin(m, q))$ with $p=0.01$ and $q=0.011$. 
    We approximate the $\ell_1$ distance by samples.
    The blue curves are the approximations and orange curves are the best line fit.}
    \label{fig:L1}
\end{figure}

Our proposed algorithm for learning Binomial distributions is given in Algorithm~\ref{alg:binom}. We compute a cover of Binomial distributions and use Algorithm~\ref{alg:phs} to select an estimate of the underlying Binomial distribution. We return the parameter of the Binomial distribution as our estimate for the underlying parameter.
Translating the guarantees on total variation distance between binomial distributions to difference of parameters requires a bound on parameter estimation from the binomial density estimation. To this end, we show the following theorem, which might be of independent interest.

\begin{theorem}
\label{thm:binomial}
For all $m$ and $p, q$, 
\[
 \ell_1(\Bin(m, p), \Bin(m,q))
= \Theta \left(\min \left(m | p - q|,  \frac{\sqrt{m} |p - q|}{ \sqrt{p(1-p)}}, 1\right) \right).
\]
\end{theorem}

Due to space constraints, we provide the proof in Appendix~\ref{app:total_var}.
We show empirically that the bounds in Theorem \ref{thm:binomial} should hold by estimating the $\ell_1$ distance between $\Bin(m, 0.01)$ and $\Bin(m, 0.011)$. Figure \ref{fig:L1} shows that the $\ell_1$ distance grows linearly with $m$ when $m$ is small, and grows linearly with $\sqrt{m}$ when $m$ is large, which illustrates our bounds in Theorem \ref{thm:binomial}.

Combining Lemma~\ref{lem:phs} with Theorem~\ref{thm:binomial} yields guarantees for Algorithm \ref{alg:binom}. Its sample complexity and utility are given by Theorem \ref{thm:main_upper}. We provide a proof in Appendix~\ref{app:main_upper}.

\begin{theorem}
\label{thm:main_upper}
Let $
s \geq \frac{16 \log(20m/\alpha\beta)}{\alpha^2} +  \frac{16 \log(20m/\alpha\beta)}{\alpha\epsilon}
$.
Given $s$ i.i.d. samples from an unknown binomial distribution $\Bin(m,p)$, Algorithm~\ref{alg:binom} returns $\hat{p}$ such that with probability at least $ 1- \beta$,
\[
|p - \hat{p}| \leq \alpha\max \left( \frac{1}{m} , \frac{\sqrt{p(1-p)}}{\sqrt{m}} \right).
\]
Furthermore, Algorithm~\ref{alg:binom} is  $(\epsilon, 0)$-differentially private.
\end{theorem}

Applying Algorithm \ref{alg:binom} independently on each symbol $i$ to learn $p_i$, we obtain Algorithm \ref{alg:dense}, an $(\epsilon, \delta)$-private algorithm that learns unknown multinomial distributions under the dense regime. Its user complexity is given by Theorem \ref{thm:ksmall}. We provide the proof in Appendix~\ref{app:ksmall}.

\begin{algorithm}[t]
\caption{Dense regime: Dense($D, \epsilon, \delta, \alpha$)}
\label{alg:dense}
\begin{algorithmic}[1]
\STATE \textbf{Input}: dataset $D$ of $s$ samples i.i.d. from $\mult(m, p)$, where $p\in\Delta_k$ , privacy parameter $\epsilon, \delta$, accuracy parameter $\alpha$.
\STATE $\epsilon'=\frac{\epsilon}{4\sqrt{k\log(1/\delta)}}$, $\alpha'=\min\left(\frac{\sqrt{m}\alpha}{2\sqrt{k}}, 1\right)$.
\STATE For each $i\in[k]$, learn the binomial distribution $\Bin(m, p_i)$ using Algorithm \ref{alg:binom}, i.e. $\hat{p}_i=\text{Binom}(D_i, \epsilon', \alpha')$, where $D_i$ is the dataset of counts of symbol $i$ in $D$.
\RETURN $\hat{p}=[\hat{p}_1, \ldots, \hat{p}_k]$.
\end{algorithmic}
\end{algorithm}

\begin{theorem}[Dense regime]
\label{thm:ksmall}
Let $k \leq m$ and $\epsilon \leq 1$. Algorithm~\ref{alg:dense} is $(\epsilon, \delta)$-differentially private and has sample complexity given by,
 \[
S^A_{m, \alpha, \epsilon, \delta} =
\cO \left(
\log \frac{km}{\alpha}\cdot 
\max \left( 
\frac{k}{m\alpha^2} + \frac{k}{\sqrt{m}\alpha\epsilon}, \frac{\sqrt{k}}{\epsilon}\sqrt{\log \frac{1}{\delta}}  \right) \right).
\]
\end{theorem}
Theorem~\ref{thm:ksmall} has a better dependency on $m$ than that of the Laplace or Gaussian mechanism.
Furthermore, even if the number of samples tends to infinity, the number of users is least $\cO(\sqrt{k})$.

\subsection{Algorithms for the sparse regime}
We now propose a more involved algorithm for the sparse regime where $m \leq k$. In this regime, users will not see samples from many symbols. A direct application of the private hypothesis selection algorithm would not yield tight bounds in this case.

We overcome this by proposing a new subroutine for estimating symbols with small probabilities, say $p_i \leq 1/m$. In this regime, most symbols appear at most once. Hence, we propose each user sends if a  symbol appeared or not i.e., $1_{N_i(u) > 0}$. Since $N_i(u)$ is distributed as $\Bin(m,p)$, observe that 
\[
\EE[1_{N_i(u) = 0}] = (1-p_i)^m.
\]

Hence, if we get a good estimate for this quantity, then since $p_i \leq 1/m$, we can use it to get a good estimate of $p_i$. 
We describe the details of this approach in Algorithm \ref{alg:psmall}. Its user complexity and utility guarantee are given by Lemma \ref{lem:small}, whose proof is in Appendix~\ref{app:small}.

\begin{algorithm}[t]
\caption{Estimation of binomial with small $p$: SmallBinom($D, \epsilon$)}
\label{alg:psmall}
\begin{algorithmic}[1]
\STATE \textbf{Input}: dataset $D$ of $s$ samples i.i.d. from $\Bin(m, p)$, privacy parameter $\epsilon$.
\RETURN $\hat{p}$ such that:
\[
(1-\hat{p})^m=\max\left(\min\left(\frac{1}{s}\sum_{u}1_{N(u)=0}+Z, 1\right), 0\right),
\]
where $Z\sim Lap(1/\epsilon)$.
\end{algorithmic}
\end{algorithm}

\begin{lemma}
\label{lem:small}
Let $p \leq \min(c/m, 1/2)$.  
Let the number of users $s \geq 64 e^{3c} \max(c, 1) \log \frac{3}{\beta}$ and $ s \geq \frac{16e^{3c}}{\alpha^2} \log \frac{3}{\beta} +  \frac{16e^{3c}}{\gamma\epsilon} \log \frac{3}{\beta}$. Algorithm~\ref{alg:psmall} is  $(\epsilon,0)$-differentially private and returns $\hat{p}$ such that with probability at least $ 1- \beta$,
\[
|p - \hat{p}|  \leq   \sqrt{\frac{p\alpha^2}{m}} + \frac{\alpha^2}{m} + \frac{\gamma}{m}.
\]
\end{lemma}

Combining the private hypothesis selection algorithm and the subroutine described in Algorithm \ref{alg:psmall}, we obtain an algorithm for the sparse regime, shown in Algorithm \ref{alg:sparse}. We first estimate $p$ using the private hypothesis selection algorithm. If for some $i$, the estimated probability is too small, we run Algorithm \ref{alg:psmall} to obtain a more accurate estimate of $p_i$. Theorem \ref{thm:klarge} gives the user complexity guarantee of Algorithm~\ref{alg:sparse}. We provide the proof in Appendix~\ref{app:klarge}.

\begin{algorithm}[t]
\caption{Sparse regime: Sparse($D, \epsilon, \delta, \alpha$)}
\label{alg:sparse}
\begin{algorithmic}[1]
\STATE Input: dataset $D$ of $s$ i.i.d. samples from $\mult(m, p), p\in\Delta_k$, privacy parameter $\epsilon, \delta$, accuracy parameter $\alpha$.
\STATE $\epsilon'=\frac{\epsilon}{8\sqrt{\min(k, m)\log\frac{1}{\delta}}}, \alpha'=\min\left(\frac{\sqrt{m}\alpha}{8\sqrt{k}}, 1\right), \alpha''=\frac{\alpha}{240}$.
\STATE $\hat{p}=\text{Dense}(D, \epsilon', \alpha'')$.
\STATE Obtain $D_1, \ldots, D_k$ from $D$ where each $D_i$ consists of $s$ i.i.d. samples from $\Bin(m, p_i)$.
\FOR{$i=1:k$} 
\IF{$\hat{p}_i<2/m$}
\STATE $\hat{p_i}\leftarrow \text{SmallBinom}(D_i, \epsilon')$, where $D_i$ is the dataset of counts of symbol $i$ in $D$.
\ENDIF
\ENDFOR
\RETURN $\hat{p}=[\hat{p}_1, \ldots, \hat{p} _k]$.
\end{algorithmic}
\end{algorithm}

\begin{theorem}
\label{thm:klarge}
Let $\epsilon \leq 1$ and $k \geq m$.
Algorithm~\ref{alg:sparse} is  $(\epsilon, \delta)$-differentially private algorithm and has sample complexity,
 \[
S^A_{m, \alpha, \epsilon, \delta} 
= \cO \left( \log \frac{km}{\alpha}  \cdot \left(
\frac{k}{m\alpha^2} + \frac{k}{\sqrt{m}\epsilon\alpha} \sqrt{\log \frac{1}{\delta}} \right) \right).
\]
\end{theorem}

\section{Conclusion}
We studied user-level differential privacy and its theoretical limit in the context of learning discrete distributions and proposed a near-optimal algorithm.

Generalizing the results to non-i.i.d. user data, proposing a more practical algorithm, and characterizing user-level privacy for other statistical estimation problems such as empirical risk minimization are interesting future research directions.
Our techniques for obtaining lower bounds on restricted differentially private estimators and the lower bound on the total variation between binomial distributions could be of interest in other scenarios.

\conf{
\section{Broader impact}

 In this work, we propose algorithms that have better privacy-utility trade-offs under global differential privacy compared to those of standard algorithms. Privacy-aware techniques are crucial for widespread use of machine learning leveraging user data. While our work is theoretical in nature, we hope that having higher utility private algorithms would encourage more practitioners to adopt user-level differential privacy in their applications.
}

\section{Acknowledgements}

Authors thank Jayadev Acharya, Peter Kairouz, and Om Thakkar for helpful comments and discussions.

\bibliographystyle{abbrvnat}
\bibliography{ref}
 
\newpage
\appendix

\conf{
\begin{center}
{\Large \textbf{Appendix: Learning discrete distributions: user vs item-level privacy}}    
\end{center}
}
\section{Proof of Lemma \ref{lem:naive}}
\label{app:laplace}

Note that $\hat{p}_i= (N_i+Z_i)/(sm)$. Thus,
\begin{align*}
    \EE[\ell_1(p, \hat{p})]&=\EE\sum_{i=1}^k|\hat{p}_i-p_i|\\
    &=\sum_{i=1}^k\EE\left|\frac{N_i+Z_i}{sm}-p_i\right|\\
    &\le \sum_{i=1}^k\EE\left|\frac{N_i}{sm}-p_i\right|+\frac{1}{sm}\EE\sum_{i=1}^k|Z_i|.
\end{align*}
The first term is upper bounded by $\sqrt{k/(sm)}$ from classic learning bounds for discrete distribution, which can be obtained by applying the Cauchy-Schwartz inequality, and noting that $N_i\sim \Bin(sm, p_i)$,
\begin{align*}
    \left(\EE\sum_{i=1}^k\left|\frac{N_i}{sm}-p_i\right|\right)^2
    &\le \EE \left[k\cdot\sum_{i=1}^k\left|\frac{N_i}{sm}-p_i\right|^2\right]\\
    &=k\sum_{i=1}^k\EE\left[\left|\frac{N_i}{sm}-p_i\right|^2\right]\\
    &=k\sum_{i=1}^k\frac{Var(N_i)}{(sm)^2}=k\cdot \sum_{i=1}^k\frac{p_i(1-p_i)}{sm}\\
    &\le k\sum_{i=1}^k\frac{p_i}{sm}=\frac{k}{sm}.
\end{align*}

For Laplace mechanism, $Z_i\sim Lap(2m/\varepsilon)$, we have $\EE |Z_i|=2m/\varepsilon$. Thus,
\[
\EE[\ell_1(p, \hat{p})]\le \sqrt{\frac{k}{sm}}+\frac{2k}{s\epsilon}.
\]

For Gaussian mechanism, $Z_i\sim N(0, \sigma^2)$ where $\sigma^2=4\log(1.25/\delta)m^2/\epsilon^2$. Using Jensen's inequality we have $\EE|Z_i|\le \sqrt{\EE[Z^2]}=\sigma$. Thus,
\[
\EE[\ell_1(p, \hat{p})]\le \sqrt{\frac{k}{sm}}+O\left(\frac{k}{s\epsilon}\sqrt{\log\frac{1}{\delta}}\right).
\]

Setting the right hand side of the above inequalities to be $\le \alpha$ and rearranging the terms we obtain the desired lower bound on $s$.

\section{Lower bounds}

\subsection{Proof of Theorem \ref{thm:Assouad}}
\label{app:assouad}
The proof of Assouad's Lemma relies on Le Cam's method \citep{LeCam73,yu1997assouad}, which provide lower bounds for min-max error in hypothesis testing. Let 
$\mathcal{P}_1\subseteq \mathcal{P}$ and $\mathcal{P}_2\subseteq\mathcal{P}$ be two disjoint subsets of distributions. Let $\hat{\theta}:\cX^s\mapsto \{1, 2\}$ be an estimator of the indices, which receives $s$ samples and predicts whether the samples come from $\mathcal{P}_1$ or $\mathcal{P}_2$. We are interested in the worst case error probability
\[
P_e(\hat{\theta}, \mathcal{P}_1, \mathcal{P}_2)=\max_{i\in\{1, 2\}}\max_{p\in\mathcal{P}_i}\Pr_{X^s\sim p^s}(\hat{\theta}(X^s)\ne i).
\]
\begin{theorem}[$(\epsilon, \delta)$-DP Le Cam's method for restricted tests] 
Let $p_1^s\in co(\mathcal{P}_1^s)$ and $p_2^s\in co(\mathcal{P}_2^s)$ where 
$co(\mathcal{P}_i^s)$ represents the convex hull of $\mathcal{P}_i^s:=\{p^s: p\in\cP_i\}$. 
Let $(X^s, Y^s)$ be an $f-$coupling between $p_1^s$ and $p_2^s$ with $\EE[d_{h}(X^s, Y^s)]=D$. 
Then for $\epsilon\ge 0, \delta\ge 0$, any $f$-restricted $(\epsilon, \delta)$-DP hypothesis testing
algorithm $\hat{\theta}$ must satisfy
\[
P_e(\hat{\theta}, \mathcal{P}_1, \mathcal{P}_2 )\ge 
\frac{1}{2}\max\{1-d_{TV}(p_1^s, p_2^s), 0.9e^{-10\epsilon D}-10D\delta\}.
\]
\label{thm:lecam}
\end{theorem}
\begin{proof}
The first term follows from the classic Le Cam's lower bound (see \cite[Lemma 1]{yu1997assouad}). For the second term, let $(X^s, Y^s)$ be an $f$-coupling of $p_1^s, p_2^s$ with $\EE[d_{h}(X^s, Y^s)]\le D$. Define $\cW:=\{(x^s, y^s)|d_{h}(x^s, y^s)\le10D \}$ as the set of realizations with Hamming distance at most $10D$. By Markov's inequality, 
\begin{equation}
    \sum_{(x^s, y^s)\notin \cW}\Pr(x^s, y^s)=\Pr(d_{h}(X^s, Y^s)>10D)<0.1
    \label{eq:markov}
\end{equation}
Let $x^s, y^s$ be the realizations of $X^s$ and $Y^s$ respectively and define 
\[\Pr(x^s, y^s):=\Pr(X^s=x^s, Y^s=y^s)
. 
\]

To avoid confusion, we let $(X')^s$ and $(Y')^s$ be random variables from $p_1^s$ and $p_2^s$ respectively. Let 
\[
\beta_1=\Pr_{(X')^s\sim p_1^s}(\hat{\theta}((X')^s)=2)
\]
be the error probability when the underlying data is from distribution $p_1^s$. Similarly define $\beta_2=\Pr_{(Y')^s\sim p_2^s}(\hat{\theta}((Y')^s)=1)$. Then
\begin{align*}
\beta_1 &= \Pr_{(X')^s\sim p_1^s}(\hat{\theta}((X')^s)=2)=\Pr(\hat{\theta}(X^s)=2) \\
&=\sum_{x^s, y^s}\Pr(X^s=x^s, Y^s=y^s)\Pr(\hat{\theta}(X^s)=p_2|X^s=x^s)\\
&\ge \sum_{x^s, y^s\in \cW}\Pr(X^s=x^s, Y^s=y^s)\Pr(\hat{\theta}(X^s)=p_2|X^s=x^s).
\end{align*}
Next we need the group property of differential privacy.
\begin{lemma}[\cite{acharya2020differentially} Lemma 18]
Let $\hat{\theta}$ be an $(\epsilon, \delta)$-DP algorithm, then for sequences $x^s, y^s\in \mathcal{X}^s$ such that $d_{h}(x^s, y^s)\le t$, we have for all subset $S$ of the output domain, 
$$\Pr(\hat{\theta}(y^s)\in S)\le e^{t\epsilon}\Pr(\hat{\theta}(x^s)\in S)+\delta te^{\epsilon(t-1)}.$$
\label{lem:group-dp}
\end{lemma}

Note that 
\[1-\beta_2 =\Pr_{(Y')^s\sim p_2^s}(\hat{\theta}((Y')^s)=2)=\Pr(\hat{\theta}(Y^s)=2).
\]
By Lemma \ref{lem:group-dp} and~\eqref{eq:markov},
\begin{align*}
    1-\beta_2&=\sum_{(x^s, y^s)\notin \cW}\Pr(x^s, y^s)\Pr(\hat{\theta}(Y^s)=2|Y^s=y^s)+\sum_{(x^ s, y^s)\in \cW}\Pr(x^s, y^s)\Pr(\hat{\theta}(Y^s)=2|Y^s=y^s)\\
    &\le0.1+\sum_{(x^s, y^s)\in \cW}\Pr(x^s, y^s)(e^{10\epsilon D}\Pr(\hat{\theta}(X^s)=2|X^s=x^s)+10D\delta e^{\epsilon(10D-1)})\\
    &\le0.1+\beta_1e^{10\epsilon D}+10D\delta e^{10\epsilon D}.
\end{align*}
Similarly we have
\[
1-\beta_1\le0.1+\beta_2e^{10\epsilon D}+10D\delta e^{10\epsilon D}.
\]
Adding the two inequalities and rearranging the terms we obtain
\[
\beta_1+\beta_2\ge\frac{1.8-10D\delta e^{10\epsilon D}}{1+e^{10\epsilon D}}\ge 0.9e^{-10\epsilon D}-10D\delta,
\]
which yields the desired lower bound. 

\end{proof}

We now have the necessary ingredients for the Assouad's lower bound. The final step is to apply the classic Assouad's Lemma \citep{yu1997assouad}:
\begin{theorem}[Assouad's Lemma] Consider a set of distributions $\mathcal{P}_\cV$ indexed by the hypercube $\mathcal{V}:=\{\pm 1\}^k$. Using the same definitions as in Theorem \ref{thm:Assouad}, 
$\forall i\in[k]$, let $\phi_i:\mathcal{X}^s\mapsto \{-1, 1\}$ be test for $p_{+i}^s$ and $p_{-i}^s$. Then for any estimator $\hat{\theta}$
\begin{equation}
    \sup_{p\in\mathcal{P}}\EE_{X^s\sim p^s}\ell(\theta(p), \hat{\theta}(X^s)) \ge \frac{\tau}{2}\sum_{i=1}^k\inf_{\phi_i}(\Pr_{X^s\sim p_{+i}^s}(\phi_i(X^s)\ne 1)+\Pr_{X^s\sim p_{-i}^s}(\phi_i(X^s)\ne -1)).
    \label{equ:classic_assouad}
\end{equation}
\end{theorem}
Note that the summand in \eqref{equ:classic_assouad} is the error probability of hypothesis testing between the mixtures $p_{+i}^s$ and $p_{-i}^s$. Applying Theorem \ref{thm:lecam} completes the proof.

\subsection{Detailed proof of Theorem \ref{thm:lb-restricted-Assouad}}
\label{app:thm1_detail}
\begin{proof}
Let $\cP_\cV$ be given by \eqref{equ:paninsky}. For $p_v\in\cP_\cV$, let $q_v=\theta(p_v)$ be the underlying discrete distribution over $k$ symbols. Then for $u, v\in \mathcal{V}$,
\[
\ell_1(\theta(p_u), \theta(p_v))=\ell_1(q_u, q_v)=\frac{12\alpha}{k}\sum_{i=1}^{k/2}\indic[u_i\ne v_i],
\]
as one different coordinate between $q_u$ and $q_v$ leads to $l_1$ distance of $12\alpha/k$. Therefore $\tau = 6\alpha/k$. Define the mixtures as
\[
p_{+i}^{s}=\frac{2}{|\mathcal{V}|}\sum_{v\in\mathcal{V}:v_i=+1}p_v^{s},\quad
p_{-i}^{s}=\frac{2}{|\mathcal{V}|}\sum_{v\in\mathcal{V}:v_i=-1}p_v^{s}.
\]
It is helpful to look at the underlying distribution of all samples from users.
\[
q_{+i}^{sm}=\frac{2}{|\mathcal{V}|}\sum_{v\in\mathcal{V}:v_i=+1}q_v^{sm},\quad
q_{-i}^{sm}=\frac{2}{|\mathcal{V}|}\sum_{v\in\mathcal{V}:v_i=-1}q_v^{sm}.
\]

Note that $p^{s}_{\pm i}, q^{s}_{\pm i}$ are not necessarily product distributions. 

By ~\cite[Lemma 14 ]{acharya2020differentially}, there exists a coupling $(U^{sm}, V^{sm})$ between $q_{+i}^{sm}$ and $q_{-i}^{sm}$ such that $\EE[d_{h}(U^{sm}, V^{sm})]\le 6\alpha sm/k$ (each $U_i, V_i\in [k]$). We construct $X^s=[X_1, ..., X_s]$ and $Y^s=[Y_1, ..., Y_s]$ using this coupling (each $X_i, Y_i\in \RR^k$ is the count of symbol $i\in[k]$). 

For each realization of $U^{sm}, V^{sm}$, suppose there are $l$ different coordinates, i.e. $d_{h}(U^{sm}, V^{sm})=l$, we move all different coordinates to the front so that only the first $\lceil l/m\rceil\le l/m+1$ users would have different data. Name the rearranged sequence as $(U')^{sm}, (V')^{sm}$. Then we let user $u$ get data from the $m(u-1)+1$ to $mu$ coordinates of $(U')^{sm}$ and $(V')^{sm}$ respectively and compute the counts of each symbol to obtain $X^s, Y^s$. Therefore,
\[
\EE[d_{h}(X^s, Y^s)]\le \frac{1}{m}\EE[d_{h}(U^{sm}, V^{sm})]+1\le \frac{6s\alpha}{k}+1.
\]
Rearranging the coordinates of $U^{sm}, V^{sm}$ would not change the total count $N$, and hence $(X^s, Y^s)$ is an $N$-coupling.
As a result.
\[
\sup_{p\in\mathcal{P}}\EE[\ell_1(p, \hat{p})]\ge 3\alpha(0.9 e^{-10\epsilon(6s\alpha/k+1)}-10\delta(6s\alpha/k+1) ).
\]
Choosing $\alpha = \min\{\frac{0.1k}{60s(\epsilon+\delta)}, \frac{1}{3}\}$ yields,
\begin{align*}
\sup_{p\in\mathcal{P}}\EE[\ell_1(p, \hat{p})] &\ge
\min\left\{ \frac{k}{200s(\epsilon+\delta)}, 1\right\}\left(0.9\exp\left\{-\frac{0.1\epsilon}{\epsilon+\delta}-10\epsilon\right\}
-\frac{0.1\delta}{\epsilon+\delta}-10\delta\right) .
\end{align*}

When $\epsilon+\delta\le 0.07$,
\begin{align*}
    \sup_{p\in\mathcal{P}}\EE[\ell_1(p, \hat{p})] &\ge \min\left\{ \frac{k}{200s(\epsilon+\delta)},1\right\}\left(0.9\left(1-\frac{0.1\epsilon}{\epsilon+\delta}-10\epsilon\right)-\frac{0.1\delta}{\epsilon+\delta}-10\delta\right)\\
    &\ge \min\left\{ \frac{k}{200s(\epsilon+\delta)},1\right\}(0.9-0.1-10(\epsilon+\delta))\\
    &\ge 0.1\min\left\{ \frac{k}{200s(\epsilon+\delta)},1\right\}.
\end{align*}

Setting the left hand side to be at most $\alpha$ and rearranging the terms, we obtain the desired lower bound for $s$.
\end{proof}

\subsection{Fano's Lower bound for restricted differentially-private estimators}
\label{app:fano}
In this section we provide learning lower bound for restricted estimators under pure differential privacy using Fano's method. First we provide a theorem for restricted estimators like the one we proposed for Assouad's, which might be of general interest.
\begin{theorem}
[$\varepsilon$-DP Fano's lower bound for restricted estimators]
\label{lem:fano-restricted}
Given a family of distributions $\mathcal{P}$ over $\mathcal{X}$ parameterized by $\theta:\mathcal{P}\mapsto \Theta$, and let $\hat{\theta}$ be an $f$-restricted estimator. Let $\mathcal{V}=\{p_1, ..., p_M\}\subseteq\mathcal{P}$ such that for all $i\ne j$,
\begin{enumerate}
    \item $\ell(\theta(p_i), \theta(p_j))\ge \alpha$
    \item $d_{KL}(p_i^s, p_j^s)\le \beta$
    \item there exists an $f$-coupling $(X^s, Y^s)$ of $p_i^s, p_j^s$ such that $\EE[d_{h}(X^s, Y^s)\le D]$
\end{enumerate}
then
\begin{align}
    L(\mathcal{P}, l, \varepsilon, 0)&:=\inf_{\hat{\theta}}\sup_{p\in \mathcal{P}}\EE_{X^s\sim p^s}\left[\ell(\hat{\theta}(X^s), \theta(p))\right] \nonumber\\
    &\ge \max\left\{\frac{\alpha}{2}\left(1-\frac{\beta+\log 2}{\log M}\right), 0.4\alpha\min\left\{1, \frac{M}{e^{10\varepsilon D}}\right\} \right\}.
    \label{equ:fano-lb}
\end{align}
\end{theorem}
\begin{proof}
The first term of (\ref{equ:fano-lb}) follows from the non-private Fano's inequality. We now prove the second term. For an observation $X^s\in\mathcal{X}^s$
\[
\hat{p}(X^s):=\arg\min_{p\in\mathcal{V}}\ell(\theta(p), \hat{\theta}(X^s))
\]
is the distribution in $\mathcal{P}$ closest to the output of our estimator. Since we require that $\hat{\theta}$ to be $\varepsilon$-DP, $\hat{p}$ is also $\varepsilon$-DP. By triangle inequality, for all $p\in\cP$
\[
\ell(\theta(\hat{p}), \theta(p))\le \ell(\theta(\hat{p}), \hat{\theta}(X^s))+\ell(\theta(p),\hat{\theta}(X^s))\le 2\ell(\theta(p),\hat{\theta}(X^s)).
\]
Thus
\begin{align}
    \sup_{p\in\mathcal{P}}\EE_{X^s\sim p^s}\left[\ell(\hat{\theta}(X^s), \theta(p))\right]&\ge \max_{p\in\mathcal{V}}\EE_{X^s\sim p^s}\left[\ell(\hat{\theta}(X^s), \theta(p))\right]\nonumber\\
    &\ge \frac{1}{2}\max_{p\in\mathcal{V}}\EE_{X\sim p}\left[\ell(\theta(\hat{p}), \theta(p))\right]\nonumber\\
    &\ge \frac{\alpha}{2}\max_{p\in\mathcal{V}}\Pr_{X\sim p}(\hat{p}(X^s)\ne p)\nonumber\\
    &\ge \frac{\alpha}{2M}\sum_{p\in\mathcal{V}}\Pr_{X\sim p}(\hat{p}(X^s)\ne p).
    \label{equ:lb-proof}
\end{align}

Let $\beta_i = \Pr_{X^s\sim p_i^s}(\hat{p}(X^s)\ne p_i)$. For a fixed $j\ne i$, let $(X^s, Y^s)$ be the $f$-coupling of $p_i^s, p_j^s$ in condition 3. By definition, for $(X')^s\sim p_i^s$, we have $
(X')^s\sim_f X^s$ so that $\hat{p}((X')^s)$ and $\hat{p}(X^s)$ have the same distributions, i.e. for all $p\in\mathcal{V}$,
$$\Pr_{(X')^s\sim p_i^s}(\hat{p}((X')^s)=p)=\Pr(\hat{p}(X^s)=p).$$ 

Same holds for $\hat{p}(Y^s)$ and $\hat{p}((Y')^s)$ such that $(Y')^s\sim p_j^s$.

By Markov's inequality, 
\[\Pr(d_{h}(X^s, Y^s)>10D)<1/10.
\]
Let $\mathcal{W}:=\{(x^s, y^s)|d_{h}(x^s, y^s)\le 10D\}$ and $\Pr(x^s, y^s):=\Pr(X^s=x^s, Y^s=y^s)$. Then
\begin{align*}
1-\beta_j &= \Pr_{(Y')^s\sim p_j^s}(\hat{p}((Y')^s)=p_j)=\Pr(\hat{p}(Y^s)=p_j)\\
&\le \sum_{(x^s, y^s)\in\mathcal{W}}\Pr(x^s, y^s)\Pr(\hat{p}(Y^s)=p_j|Y^s=y^s)+\sum_{(x^s, y^s)\notin\mathcal{W}}\Pr(x^s, y^s)\cdot 1  .
\end{align*}
 Therefore
\[
\sum_{(x^s, y^s)\in\mathcal{W}}\Pr(x^s, y^s)\Pr(\hat{p}(Y^s)=p_j|Y^s=y^s)\ge 0.9-\beta_j.
\]
Furthermore
\begin{align*}
    \Pr_{(X')^s\sim p_i^s}(\hat{p}((X')^s)=p_j)&=\Pr(\hat{p}(X^s)=p_j)\\
    &\ge \sum_{(x^s, y^s)\in \mathcal{W}}\Pr(x^s, y^s)\Pr(\hat{p}(X^s)=p_j|X^s=x^s)\\
    &\ge \sum_{(x^s, y^s)\in \mathcal{W}}\Pr(x^s, y^s)e^{-10\varepsilon D}\Pr(\hat{p}(Y^s)=p_j|Y^s=y^s)\\
    &\ge (0.9-\beta_j)e^{-10\varepsilon D},
\end{align*}
where the second inequality is due to $\hat{p}$ is $\varepsilon$-DP and $d_{h}(x^s, y^s)\le 10D$. The above inequality holds for all $j\ne i$. Thus summing over all $j\ne i$ we obtain
\[
\beta_i=\sum_{j\ne i}\Pr_{X^s\sim p_j^s}(\hat{p}(X^s)=p_j)\ge \left(0.9(M-1)-\sum_{j\ne i}\beta_j\right)e^{-10\varepsilon D}.
\]
Summing over all $i\in\{1, ..., M\}$
\[
\sum_{i=1}^M\beta_i\ge \left(0.9M(M-1)-(M-1)\sum_{i=1}^M\beta_i)\right)e^{-10\varepsilon D}.
\]
Rearranging the terms
\[
\sum_{i=1}^M\beta_i\ge \frac{0.9 M(M-1)}{M-1+e^{10\varepsilon D}}\ge 0.8M\min\left\{1, \frac{M}{e^{10\varepsilon D}}\right\}.
\]
Combining with (\ref{equ:lb-proof}) gives the desired lower bound. 
\end{proof}  

\begin{proof}[Proof of Theorem \ref{thm:lowerall}]
We apply Theorem \label{fano:restricted} with $f$ as the identity mapping. In this case it is the same as \cite[Theorem 2]{acharya2020differentially}.

Assume $k$ is even. From~\cite{yu1997assouad}, there exists $\cV\subseteq\{-1, 1\}^{k/2}$ and a universal $c_0>0$ such that $|\mathcal{V}|\ge \exp(c_0k/2)$, each pair at least $k/6$ apart in Hamming distance.
Given $\alpha\in(0, 1/6)$, define a family of multinomial distributions $\mathcal{P}_\nu$ which consists of the following distributions indexed by $\nu=(\nu_1, ..., \nu_{k/2})\in\mathcal{V}$,
\[
p_\nu = \mult\left(m, \frac{1}{k}(1+3\alpha \nu_1, 1-3\alpha\nu_1, ..., 1+3\alpha\nu_{k/2}, 1-3\alpha\nu_{k/2})\right).
\]
For $v\in\cV$, let $q_v=\theta(p_v)$ be the underlying $k$-ary distribution. Thus for each pair of distributions $p_u, p_v$ from this family we have $\ell_1(\theta(p_u), \theta(p_v))=\ell_1(q_u, q_v)\ge 12\alpha/k\cdot k/6=2\alpha$. Furthermore,
\[
d_{KL}(q_u||q_v)\le \chi^2(q_u||q_v)=\sum_{x=1}^k\frac{(q_u(x)-q_v(x))^2}{q_v(x)}\le 100\alpha^2,
\]
\[
d_{KL}(p_u||p_v)=md_{KL}(q_u||q_v)\le 100m\alpha^2,
\]
\[
d_{KL}(p_u^s||p_v^s)=sd_{KL}(p_u||p_v)\le 100sm\alpha^2.
\]

Since $f$ is set to be the identity, we just need to design a coupling with appropriate Hamming distance for each pair $p_u^{s}, p_v^{s}, u, v\in\nu$. To this end we need the following lemma from \cite{Hollander12}.
\begin{lemma}[Maximal coupling, \cite{Hollander12}]
Given distributions $q_1, q_2$ over some domain $\mathcal{X}$, there exists a coupling $(X^s, Y^s)$ between $q_1^s$ and $q_2^s$ such that 
\[
\EE[d_{h}(X^s, Y^s)]=s\cdot d_{TV}(q_1, q_2).
\]
\label{lem:max-coupling}
\end{lemma}

From Lemma \ref{lem:max-coupling} there exists a coupling $(X^s, Y^s)$ between $p_u^{s}$ and $p_v^{s}$ such that 
\[
\EE[d_{h}(X^s, Y^s)]=s\cdot d_{TV}(p_u, p_v).
\]
Using Pinsker's inequality, we have
\[
d_{TV}(p_u, p_v)\le \sqrt{\frac{1}{2}d_{KL}(p_u||p_v)}\le 10\sqrt{m}\alpha .
\]
Therefore $\EE[d_{h}(X^s, Y^s)]\le 10s\sqrt{m}\alpha$. Applying Lemma \ref{lem:fano-restricted} yields,
\[
\sup_{p\in\mathcal{P}}\EE[\ell_1(\hat{p}, p)]\ge \max\left\{\alpha\left(1-\frac{100sm\alpha^2+\log 2}{c_0k/2}\right), 0.8\alpha\min\left\{1, \frac{e^{c_0k/2}}{e^{100\varepsilon s\sqrt{m}\alpha}}\right\}\right\}.
\]
Note that this holds for all $\alpha$. Choose $\alpha=\min\{\frac{1}{6}, \sqrt{\frac{k}{sm}}\}$ and $\alpha=\min\{\frac{1}{6},\frac{c_0k}{200s\sqrt{m}\epsilon} \}$ respectively we get
\[
\sup_{p\in\mathcal{P}}\EE[\ell_1(\hat{p}, p)]\ge\max\left\{C_1\sqrt{\frac{k}{sm}}, C_2\frac{k}{s\varepsilon}\right\}=\Omega\left(\sqrt{\frac{k}{sm}}+\frac{k}{s\sqrt{m}\varepsilon}\right).
\]
Given desired accuracy $\alpha$, setting $\sup_{p\in\mathcal{P}}\EE[\ell_1(\hat{p}, p)]\le \alpha$ gives the desired user complexity bound.
\end{proof}

\section{Bounds on total variation between binomial distributions}
\label{app:total_var}

We divide the proof of Theorem~\ref{thm:binomial} into two parts. We prove the upper bound in Lemma~\ref{lem:binom-dtv-ub} and the lower bound in Lemma~\ref{lem:binom-dtv-lb}.

We first prove an upper bound on the 
total variation distance between 
binomial distributions in terms of the parameters.
\begin{lemma}
\label{lem:binom-dtv-ub}
There is a constant $b$ such that 
for all $m$ and $p, q$,
\[
\ell_1(\Bin(m, p), \Bin(m,q))
\leq 2 \min \left(m | p - q|,  \frac{\sqrt{m} |p - q|}{ \sqrt{p(1-p)}}, 1\right).
\]
\end{lemma}
\begin{proof}
First observe that by definition,
\begin{equation}
    \label{eq:up1}
\ell_1(\Bin(m, p), \Bin(m,q)) \leq 2.
\end{equation}
Secondly, since $\ell_1$ distance of product distributions is at most the sum of $\ell_1$ distances,
\begin{equation}
    \label{eq:up2}
\ell_1(\Bin(m, p), \Bin(m,q)) \leq 
m \cdot \ell_1(\text{Ber}(p), \text{Ber}(q))
\leq 2 m | p -q|.
\end{equation}
Finally, by Pinkser inequality and the fact that KL divergence of product distributions is the sum of individual KL divergences,
\begin{align}
\ell_1(\Bin(m, p), \Bin(m,q))
& \leq \sqrt{\frac{1}{2} \cdot D(\Bin(m, q) || \Bin(m, p))} \nonumber \\
& =  \sqrt{\frac{m}{2} \cdot D(\text{Ber}(q) || \text{Ber}(p))} \nonumber \\
& \leq \sqrt{\frac{m(p-q)^2}{2p(1-p)}}, \label{eq:up3}
\end{align}
where the last inequality follows by observing that 
\begin{align}
D(\text{Ber}(q) || \text{Ber}(p))
& = q \log \frac{q}{p} + (1-q) \log \frac{1-q}{1-p}  \nonumber \\
& = q \log \left(1 + \frac{q - p}{p}\right) + (1-q) \log \left(1 + \frac{p-q}{1-p} \right) \nonumber \\
& \leq  q \cdot \frac{q-p}{p}  + (1-q) \cdot \frac{p-q}{1-p}  \nonumber \\
& = \frac{(q-p)^2}{p(1-p)} \label{eq:chi}.
\end{align}
Combining~\eqref{eq:up1},~\eqref{eq:up2},and~\eqref{eq:up3} yields the lemma.
\end{proof}

\begin{lemma}
\label{lem:binom-dtv-lb_low}
Let $c$ be a constant. If $mp< c$ and $p \leq 1/2$, then 
\[
\ell_1(\Bin(m, p), \Bin(m,q))
\geq \frac{e^{-\frac{3c}{2}}}{2} \min \left(m | p -q|, 1 \right).
\]
\end{lemma}
\begin{proof}
By definition,
\[
\ell_1(\Bin(m, p), \Bin(m,q))
\geq |(1-p)^m - (1-q)^m|.
\]
We first consider the case $q \geq p$.
Simplifying the above bound,
\begin{align*}
    (1-p)^m - (1-q)^m
    & = (1-p)^m \left( 1- \frac{(1-q)^m}{(1-p)^m} \right) \\
    & = (1-p)^m \left( 1- \left(1 - \frac{q-p}{1-p} \right)^m\right) \\   & \stackrel{(a)}{\geq} (1-p)^m \left( 1- e^{\frac{-m(q-p)}{1-p}} \right) \\
      & \stackrel{(b)}{\geq} (1-p)^m \left( 1- e^{-2m(q-p)} \right) \\
        & {\geq} (1-p)^m \left( 1- e^{-1.5m(q-p)} \right) \\
        & {\geq} (1-p)^m \left( 1- e^{-1.5\min(m(q-p),0.5)} \right) \\
      & \stackrel{(c)}{\geq} (1-p)^m \min(m(q-p),0.5) \\
          & \stackrel{(d)}{\geq} e^{-1.5mp} \min(m(q-p),0.5) \\
            & \stackrel{(e)}{\geq} e^{-1.5c} \min(m(q-p),0.5).  \\
\end{align*}
$(a)$ follows by $1-x \leq e^{-x}$ and
$(b)$ follows as $p \leq 1/2$. $(c)$ and $(d)$ follows as $e^{-1.5x} \leq 1 - x $ for $x \leq 1/2$. $(e)$ follows by the bound on $p$. For $q \leq p$,
\begin{align*}
   (1-q)^m -  (1-p)^m 
    & = (1-p)^m \left(\frac{(1-q)^m}{(1-p)^m} - 1  \right) \\
       & = (1-p)^m \left(\left( 1 + \frac{p-q}{1-p}\right)^m - 1  \right) \\
         & \geq (1-p)^m \left(\left( 1 + p-q\right)^m - 1  \right) \\
            & \stackrel{(a)}{\geq} (1-p)^m m(p-q)\\
                & \ge e^{-1.5mp}m(p-q)\\
                    & \ge e^{-1.5c}m(p-q),
\end{align*}
$(a)$ follows from the Bernoulli inequality: $(1+x)^n\ge 1+nx$ for $x\ge -1$. The last inequalities are similar to the last two inequalities for $q \leq p $ case. Combining the above two results, we get 
\begin{equation}
    \label{eq:p_low}
    | (1-q)^m -  (1-p)^m |
    \geq e^{-1.5c} \min(m|q-p|,0.5).
\end{equation}
\end{proof}
\begin{lemma}
\label{lem:binom-dtv-lb_high}
Let $c > 2$, $m \geq 3$, and $p \leq 1/2$. If $mp \geq c$, then 
\[
\ell_1(\Bin(m, p), \Bin(m,q))
\geq \frac{1}{350}  \min \left(\frac{\sqrt{m}| p - q|}{\sqrt{p(1-p)}}, 1   \right).
\]
\end{lemma}
\begin{proof}
Let $q' = p + \sqrt{\frac{p}{8m}}$ if $q >  p + \sqrt{\frac{p}{8m}}$, $q' = p - \sqrt{\frac{p}{8m}}$ if $q \leq   p - \sqrt{\frac{p}{8m}}$, else $q' = q$. Since $q'$ lies in between $p$ and $q$,
\[
\ell_1(\Bin(m, p), \Bin(m,q))
\geq \ell_1(\Bin(m, p), \Bin(m,q')).
\]
Furthermore, observe that 
\[
\frac{3}{4} \leq 1 - \frac{1}{\sqrt{8c}} \leq 1 - \sqrt{\frac{1}{8pm}} \leq \frac{q'}{p} \leq 1 + \sqrt{\frac{1}{8pm}} \leq 1 + \frac{1}{\sqrt{8c}} \leq \frac{5}{4}.
\]
By~\cite[Proposition 2.3]{adell2006exact}, for any two binomial distributions,
\[
\ell_1(\Bin(m, p), \Bin(m,q'))
= m\int^{\max(p,q')}_{u = \min(p, q')}  \Pr(\Bin(m - 1, u) = k-1) du,
\],
where $\lceil m\min(p,q') \rceil \le k  \leq \lceil m\max(p,q') \rceil $. Furthermore, observe that 
\[
\lceil m\min(p,q') \rceil \geq \lceil m  \min(mp,3mp/4) \rceil \geq \lceil 3/2 \rceil \geq 2.
\]
Similarly,
\[
m - k \geq m - \lceil m\max(p,q') \rceil
 \geq m -  \lceil 5mp/4 \rceil
 \geq m - 1 - 5mp/4 \geq m - 1 - 5m/8
 \geq 3m/8 - 1 \geq 1/8.
\]
Since $m-k$ is an integer, $m-k \geq 1$.
In order to bound the above quantity further,
we first lower bound Binomial coefficients.

\begin{align*}
 \Pr(\Bin(m, p) = k) 
& = \binom{m}{k}
p^k (1-p)^{m-k}.
\end{align*}
Recall that by Sterling's approximation, for all $m \geq 1$,
\[
\sqrt{2\pi} m^{m+0.5} e^{-m} \leq m! \leq e m^{m+0.5} e^{-m}.
\]
Hence for $1 \leq k \leq m - 1$,
\begin{align*}
    \binom{m}{k}
    & = \frac{m!}{k!(m-k)!} \\
    & \geq \frac{\sqrt{2\pi}}{e^2} \frac{ m^{m+0.5} e^{-m}}{ k^{k+0.5} e^{-k}  (m-k)^{m-k+0.5} e^{-m+k}} \\
    & = \frac{\sqrt{2\pi}}{e^2\sqrt{m}} \cdot \frac{1}{\sqrt{k/m}{\sqrt{1-k/m}}} \cdot \frac{1}
    {(k/m)^k (1-k/m)^{m-k} }.
\end{align*}
Hence,
\begin{align*}
   \Pr(\Bin(m, p) = k) 
 &  \geq \frac{\sqrt{2\pi}}{e^2\sqrt{m}} \cdot \frac{1}{\sqrt{k/m}{\sqrt{1-k/m}}} \cdot \frac{p^k (1-p)^{m-k}}{(k/m)^k (1-k/m)^{m-k}} \\
  &  = \frac{\sqrt{2\pi}}{e^2\sqrt{m}} \cdot  \frac{1}{\sqrt{k/m}{\sqrt{1-k/m}}} \cdot e^{-m D(k/m|| p)} \\
   &  \geq \frac{\sqrt{2\pi}}{e^2\sqrt{m}} \cdot  \frac{1}{\sqrt{k/m}{\sqrt{1-k/m}}} \cdot e^{-m \frac{(k/m -p)^2}{p(1-p)}} \\
&    \geq \frac{\sqrt{2\pi}}{e^2} \cdot  \frac{1}{\sqrt{k}} \cdot e^{-m \frac{(k/m -p)^2}{p(1-p)}}.
\end{align*}
The second inequality follows from~\eqref{eq:chi}. Hence for $\lceil m\min(p,q') \rceil \le k  \leq \lceil m\max(p,q') \rceil $,
\begin{align*}
  \Pr(\Bin(m, u) = k-1) 
  & \geq \frac{\sqrt{2\pi}}{e^2} \cdot  \frac{1}{\sqrt{k-1}} \cdot e^{-m \frac{((k-1)/(m-1) -u)^2}{u(1-u)}} \\
      & \stackrel{(a)}{\geq} \frac{2\sqrt{2\pi}}{5e^2} \cdot  \frac{1}{\sqrt{mp}} \cdot e^{-m \frac{((k-1)/(m-1) -u)^2}{u(1-u)}} \\
      & \geq \frac{2\sqrt{\pi}}{5e^2} \cdot  \frac{1}{\sqrt{mp(1-p)}} \cdot e^{-m \frac{((k-1)/(m-1) -u)^2}{u(1-u)}},
\end{align*}
where $(a)$ follows by observing that 
$k - 1 \leq \lceil m\max(p,q') \rceil - 1  \leq  m\max(p,q') \leq 5mp/4$. Furthermore, since $3p/4 \leq q' \leq 5p/4$ and the minimum of $u(1-u)$ occurs in the extremes,
\begin{align*}
\min_{\min(p,q')\leq u \leq \max(p,q')} u(1-u) 
& \geq 
\min_{3p/4 \leq u \leq5p/4} u(1-u) \\
& \geq 
\min(\frac{(1-3p/4)3p}{4}, \frac{(1-5p/4)5p}{4}) \\
& \geq \frac{15p}{32}.
\end{align*}
We now bound $((k-1)/(m-1) -u)^2$.
\begin{align*}
\max_{u} \frac{k-1}{m-1} - u
 \leq  
\frac{k}{m} - u  
 \leq \max(p,q') + \frac{1}{m} - \min(p,q').
\end{align*}
Similarly,
\begin{align*}
\min_{u}  \frac{k-1}{m-1} - u
& \geq
\frac{k-1}{m-1} - \min(p,q')\\
& = \frac{k}{m} + \frac{m-k}{m(m-1)}- \max(p,q') \\
& \geq \frac{k}{m} + \frac{1}{m}- \max(p,q') \\
& \geq  \min(p,q') + \frac{1}{m} - \max(p,q').
\end{align*}
Hence, since $(a + b)^2 \leq 2a^2 + 2b^2$,
\begin{align*}
\max_u \left(\frac{k}{m} - u\right)^2
& \leq 2 \left( \max(p,q') - \min(p,q')\right)^2 + \frac{2}{m^2}.
\end{align*}
Hence,
\[
e^{-m \frac{((k-1)/(m-1) -u)^2}{u(1-u)}}
\geq e^{-\frac{8m}{p} \left(\frac{1}{m^2} + (p-q')^2\right) }
\geq e^{-\frac{64m}{15p} \left(\frac{1}{m^2} + \frac{p}{8m}\right) } \geq 
e^{-\frac{32}{15}- \frac{8}{15}} \geq e^{-8/3}.
\]
Combining the results, we get
\begin{align*}
    \ell_1(\Bin(m, p), \Bin(m,q'))
& = m\int^{\max(p,q')}_{u = \min(p, q')}  \Pr(\Bin(m - 1, u) = k-1) du \\
& \geq \frac{m\sqrt{\pi}e^{-8/3}}{2e^2}\int^{\max(p,q')}_{u = \min(p, q')} \frac{m}{\sqrt{m p (1-p)}} \\
& \geq \frac{\sqrt{\pi}e^{-8/3}}{2e^2} \frac{\sqrt{m}| p - q'|}{\sqrt{p(1-p)}}  \\
&  \geq \frac{\sqrt{\pi}e^{-8/3}}{2e^2} \min \left(\frac{\sqrt{m}| p - q|}{\sqrt{p(1-p)}}, \frac{1}{\sqrt{8}}   \right)\\
&  \geq \frac{\sqrt{\pi}e^{-8/3}}{2\sqrt{8}e^2} \min \left(\frac{\sqrt{m}| p - q|}{\sqrt{p(1-p)}}, 1   \right)\\\
& \geq \frac{1}{350}  \min \left(\frac{\sqrt{m}| p - q|}{\sqrt{p(1-p)}}, 1   \right).
\end{align*}
\end{proof}
\begin{lemma}
\label{lem:binom-dtv-lb}
For all $m$ and $p, q$,
\[
\ell_1(\Bin(m, p), \Bin(m,q))
\geq \frac{1}{350} \min \left(m | p - q|,  \frac{\sqrt{m} |p - q|}{ \sqrt{p(1-p)}}, 1\right).
\]
\end{lemma}
\begin{proof}
For $m \leq 700$,
\[
\ell_1(\Bin(m, p), \Bin(m,q))
\geq 
\ell_1(\text{Ber}(p), \text{Ber}(q))
= 2|p-q| \geq \frac{1}{350} \min \left(m | p - q|,  \frac{\sqrt{m} |p - q|}{ \sqrt{p(1-p)}}, 1\right),
\]
Hence, in the rest of the proof, we focus on $m \geq 700$. Furthermore, 
since 
\[
\ell_1(\Bin(m, p), \Bin(m,q)) = \ell_1(\Bin(m, 1-p), \Bin(m,1-q)).
\]
and the bound $\frac{1}{350} \min \left(m | p - q|,  \frac{\sqrt{m} |p - q|}{ \sqrt{p(1-p)}}, 1\right)$ is symmetric in $p$ and $1-p$, it suffices to prove the result for $p \leq 1/2$.

Let $c=2$. The proof for $mp \geq c$
is a direct consequence of Lemma~\ref{lem:binom-dtv-lb_high}. The proof for $c \leq 2$ follows from Lemma~\ref{lem:binom-dtv-lb_low}.
\end{proof}

\section{Analysis of the algorithms}

\subsection{Proof of Theorem~\ref{thm:main_upper}}
\label{app:main_upper}
We first state the following guarantee on private hypothesis selection from~\cite{bun2019private}.

\begin{lemma}[\cite{bun2019private}]
\label{lem:phs}
Given $d$ distributions $p_1, p_2, \ldots, p_d$ and $n$ independent samples from an unknown distribution $p$, such that $\min_{i} \ell_1(p_i, p) \leq \alpha$, Algorithm~\ref{alg:phs} returns a distribution $p_i$ such that 
$
\EE[\ell_1(p_i, p)] \leq 4 \alpha,
$
with probability $\geq 1- \beta$, if the number of samples satisfies,
\[
n \geq \frac{8 \log(4m/\beta)}{\alpha^2} +  \frac{8 \log(4m/\beta)}{\alpha\epsilon}.
\]
Furthermore, Algorithm~\ref{alg:phs} is $(\epsilon, 0)$-differentially private.
\end{lemma}
\begin{proof}
 
The privacy guarantee follows by \cite[Lemma 3.2]{bun2019private}. The utility guarantee is obtained by applying the high probability utility bounds from \cite[Lemma 3.3]{bun2019private} and setting $\zeta = 1$.
\end{proof}

Let $c$ be the constant in the lower bound of Theorem~\ref{thm:binomial}.
Let $\cP =
\{0,\frac{c\alpha}{20m},\frac{2c\alpha}{20m},\ldots, 1\rceil \}$ be a cover of $[0,1]$ Note that such that for every $p$, there exists a $p' \in \cP$ such that
\[
\min \left(m | p - p'|,  \frac{\sqrt{m} |p - p'|}{ \sqrt{p(1-p)}}, 1\right) 
\leq \frac{c\alpha}{10}.
\]
Let $\cQ = \{\Bin(m,p) : p \in \cP\} $.
Then by Theorem~\ref{thm:binomial}, for every $\Bin(m,p)$ there exists a $\Bin(m,p')$ in $\cQ$ such that 
\[
\ell_1(\Bin(m,p), \Bin(m,p')
\leq  \frac{c\alpha}{5}.
\]
Hence, by Lemma~\ref{lem:phs}, if 
\[
s = \Omega \left( \frac{8 \log(20m/\alpha\beta)}{\alpha^2} +  \frac{8 \log(20m/\alpha\beta)}{\alpha\epsilon} \right)
\]
there is an algorithm that returns a distribution $\Bin(m,\hat{p}) \in \cQ$ such that 
\[
\ell_1(\Bin(m,p), \Bin(m,\hat{p})
\leq  \frac{4c\alpha}{5},
\]
with probability $\geq 1 - \beta$.
Therefore, by the lower bound in Theorem~\ref{thm:binomial}, the resulting $\hat{p}$ satisfies
\[
\min \left(m | p - \hat{p}|,  \frac{\sqrt{m} |p - \hat{p}|}{ \sqrt{p(1-p)}}, 1\right) \leq \frac{4\alpha}{5},
\]
with probability $\geq 1- \beta$. Since $\frac{4\alpha}{5} \leq 1$, this implies that with probability   $\geq 1- \beta$,
\[
|p - \hat{p}| \leq \frac{4\alpha}{5} \max \left( \frac{1}{m} , \frac{\sqrt{p(1-p)}}{\sqrt{m}} \right).
\]
The expectation bound follows by setting $\beta = \alpha/5m$:
\[
\EE[|p - \hat{p}|]
\leq  \frac{4\alpha}{5} \max \left( \frac{1}{m} , \frac{\sqrt{p(1-p)}}{\sqrt{m}} \right) + \frac{\alpha}{5m}
\leq \alpha \max \left( \frac{1}{m} , \frac{\sqrt{p(1-p)}}{\sqrt{m}} \right).
\]

\subsection{Proof of Theorem~\ref{thm:ksmall}}
\label{app:ksmall}
Let $\epsilon' = \frac{\epsilon}{4\sqrt{k \log \frac{1}{\delta}}}$ and $\alpha' =  \min \left( \frac{\sqrt{m}\alpha}{2\sqrt{k}}, 1\right)$ We apply Theorem~\ref{thm:main_upper} for each 
symbol $k$ with $\epsilon = \epsilon'$ and $\alpha = \alpha'$ Then, we have an estimate $\hat{p}_1, \hat{p}_2,\ldots, \hat{p}_k$ such that
\begin{align*}
\EE[\ell_1(p,\hat{p})] 
& = \sum_i \EE[|p_i-\hat{p}_i|] \\ 
& \leq  \alpha' \sum_i \max \left( \frac{1}{m} , \frac{\sqrt{p_i(1-p_i)}}{\sqrt{m}} \right) \\
& \leq  \alpha' \sum_i  \frac{1}{m} +  \frac{\sqrt{p_i}}{\sqrt{m}}  \\
& \leq    \frac{\alpha' k}{m} + 
\frac{\alpha' \sqrt{k}}{\sqrt{m}} \\
& \leq    2
\frac{\alpha' \sqrt{k}}{\sqrt{m}} \\
& \leq \alpha,
\end{align*}
where the penultimate follows from Jensen's inequality. The differential privacy bound follows from strong composition theorem (see \cite[Theorem 3.4]{kairouz2017composition}) and using the fact that $e^{\epsilon'} \leq 2\epsilon'$.

\subsection{Proof of Lemma~\ref{lem:small}}
\label{app:small}

Let $\hat{p}$ be such that 
\begin{equation}
    (1-\hat{p})^m = \max \left( \min \left(\frac{1}{s} \sum_{u} 1_{N(u) = 0} + \frac{Z}{s}, 1 \right), 0 \right),
\label{eq:sparse-p_hat}
\end{equation}

Where $Z$ is a Laplace noise with parameter $1/\epsilon$. Hence the algorithm is $(\epsilon, 0)$-DP. Hence,
\[
|(1-\hat{p})^m - (1-{p})^m| \leq 
\left \lvert \frac{1}{s} \sum_{u} 1_{N(u) = 0} + \frac{Z}{s} - (1-p)^m \right \rvert.
\]
Hence, by the tail bounds of the Laplace distribution, with probability $\geq 1- 2\beta$,
\[
|(1-\hat{p})^m - (1-{p})^m| \leq 
\frac{\log \frac{1}{\beta}}{s \epsilon} +
\left \lvert \frac{1}{s} \sum_{u} 1_{N(u) = 0}  - (1-p)^m \right \rvert.
\]
Furthermore, by Bernstein's inequality with probability $\geq 1- 2\beta$,
\[
\left \lvert \frac{1}{s} \sum_{u} 1_{N(u) = 0}  - (1-p)^m\right \rvert  \leq
4\frac{\log \frac{1}{\beta}}{s} 
+ 4\sqrt{\frac{\log \frac{1}{\beta}}{s} \cdot (1-p)^m (1- (1-p)^m)}.
\]
Since $1- (1-p)^m \leq mp$, we have with probability $\geq 1 - 4 \beta$,
\[
|(1-\hat{p})^m - (1-{p})^m| \leq 
 4\sqrt{\frac{mp \log  \frac{1}{\beta}}{s}}
 + 4 \frac{\log \frac{1}{\beta}}{s} + 
 \frac{\log \frac{1}{\beta}}{s\epsilon}. 
\]
Combining with~\eqref{eq:p_low}, with probability $\geq 1- 4\beta$,
\[
e^{-1.5c} \min(m|\hat{p} -p|,0.5)
\leq 
 4\sqrt{\frac{mp \log \frac{1}{\beta}}{s}}
 + 4\frac{\log \frac{1}{\beta}}{s} + 
 \frac{\log \frac{1}{\beta}}{s\epsilon}.
\]
If $s \geq 64 e^{3c} m \log \frac{3}{\beta}$, then the RHS is at most $e^{-1.5c}/2$. hence,
\[
e^{-1.5c} m|\hat{p} -p|
\leq 
 4\sqrt{\frac{mp \log \frac{1}{\beta}}{s}}
 + 4\frac{\log \frac{1}{\beta}}{s} + 
 \frac{\log \frac{1}{\beta}}{s\epsilon}.
\]
If $ s \geq \frac{16e^{3c}}{\alpha^2} \log \frac{3}{\beta} +  \frac{16e^{3c}}{\gamma\epsilon} \log \frac{3}{\beta}$
\[
|p - \hat{p}| \leq   \sqrt{\frac{p\alpha^2}{m}} + \frac{\alpha^2}{m} + \frac{\gamma}{m}.
\]

\subsection{Proof of Theorem~\ref{thm:klarge}}
\label{app:klarge}

\textbf{Parameters}: We first define few parameters. Let $\epsilon' = \frac{\epsilon}{8\sqrt{\min(k,m) \log \frac{1}{\delta}}}$, $\beta = \frac{\alpha}{40k}$, $\alpha' =  \min \left( \frac{\sqrt{m}\alpha}{8\sqrt{k}}, 1\right)$, $\alpha'' = \frac{\alpha}{240}$, and $\gamma = \frac{m \alpha}{8k}$. Let $c = 4/m$.

\textbf{Algorithm}:
For every symbol we first calculate the 
probability using the algorithm in Theorem~\ref{thm:main_upper} with $\epsilon = \epsilon'$, $\alpha = \alpha''$ and error probability $\beta$. If the estimated probability is less than $2/m$, we use the algorithm from Lemma~\ref{lem:small} with $\epsilon = \epsilon'$, $\alpha = \alpha'$, $\gamma = \gamma$, and error probability $\beta$.
Let ${p}'$ be the output of the first step and the $p''$ be the output of Lemma~\ref{lem:small}.
The error of the algorithm is 
\[
|p - \hat{p}| = | p - {p}'| 1_{{p}' > 2/m} + | p - {p}''|
1_{{p}' \leq 2/m}.
\]

\textbf{Sample complexity}: The sample complexity would be the sum of sample complexities of Theorem ~\ref{thm:main_upper} and Lemma~\ref{lem:small} with appropriate parameters.
Hence,
\begin{align*}
s \geq \frac{16 \log(20m/\alpha''\beta)}{\alpha''^2} +  \frac{16 \log(20m/\alpha''\beta)}{\alpha''\epsilon'} + \frac{16e^{3c}}{\alpha'^2} \log \frac{3}{\beta} +  \frac{16e^{3c}}{\gamma\epsilon'} \log \frac{3}{\beta}.
\end{align*}
Hence, for a sufficient large constant $b$, if 
\begin{align*}
s \geq b \log \frac{km}{\alpha} \cdot \left(
\frac{k}{m\alpha^2} + \frac{k}{\sqrt{m}\epsilon\alpha}  \sqrt{\log \frac{1}{\delta}}
\right)
\end{align*}
Note that since $k \geq m$, the above bound implies that $s \geq b \sqrt{m}$, hence the bound also satisfies conditions in Lemma~\ref{lem:small}. 

\textbf{Differential privacy:} We first provide the privacy guarantee for this algorithm. First observe that
since $p', p'' \to \hat{p}$ is a Markov chain, by the postprocessing theorem it suffices to provide privacy guarantee for releasing $p', p''$. Consider releasing one of them, say $p'$. For any two neighboring datasets differ in at most $\min(m,k)$ symbols. Let these datasets be $D$ and $D'$ and $S(D, D')$ be the set of symbols where they differ. For these datasets,
\[
\frac{\Pr(p' | D)}{\Pr(p'| D')}
= \prod_{i \in S(D, D')} \frac{\Pr(p'_i | D)}{\Pr(p' _i | D')}.
\]
Hence it suffices to apply strong composition theorem for this subset of size $\min(m,k)$ and the rest of the proof is similar to that of  \cite[Theorem 3.4]{kairouz2017composition}). The proof is similar for $p''$ and hence the result.

\textbf{Utility:} 
To analyze the utility, we divide the symbols into three sets $A_1 =  \{i : p_i \geq \frac{4}{m}\}$, $A_2 = \{i : \frac{4}{m} \geq p_i \geq \frac{1}{4m}\}$,
and $A_3 = \{i :  p_i \leq \frac{1}{4m}\}$. 

\textbf{Utility-large:} Consider the set $A_1$ with symbols whose probability is greater than $4/m$, for such a symbol, by Theorem~\ref{thm:main_upper}, with probability $\geq 1- \beta$,
\[
| p -p'| \leq \alpha'' \sqrt{\frac{p}{m}}.
\]
Hence $p' \geq p -  \alpha'' \sqrt{\frac{p}{m}} > \frac{2}{m}$. Hence, for such a symbol with probability $\geq 1- \beta$,
\[
|p - \hat{p}| = | p -p'| \leq   \alpha'' \sqrt{\frac{p}{m}}.
\]
\textbf{Utility-medium:} Consider the set $A_2$ with symbols whose probability in $[1/4m, 4/m]$. For such a symbol,
then with probability $\geq 1- 2\beta$,
\begin{align*}
|p - \hat{p}|
& \leq  \max(|p - p'|, |p - p''|)  \\
& \leq \frac{2\alpha''}{m} + \alpha'' \sqrt{\frac{p}{m}} +  \alpha'\sqrt{\frac{p}{m}}  + \frac{\alpha'^2}{m} + \frac{\gamma}{m} \\
& \leq \frac{5\alpha''}{m}  + \frac{\alpha'}{m} + \frac{\gamma}{m} .
\end{align*}
\textbf{Utility-small:} Finally consider symbols whose probabilities are smaller than $1/4m$, for these symbols, with probability $\geq 1- \beta$,
\[
| p -p'| \leq \frac{\alpha''}{2m}.
\]
and hence $p' \leq p +   \frac{\alpha''}{m} \leq 3/2m \leq 2/m$. Hence only the second algorithm is used. Hence with probability $\geq 1- 2\beta$,
the error is at most,
\[
| p -\hat{p}| = | p -p''| 
\leq \alpha'\sqrt{\frac{p}{m}} + \frac{\alpha'^2}{m} + \frac{\gamma}{m}.
\]
Summing over all symbols yield,
\begin{align*}
\ell_1(p, \hat{p})
& \leq \sum_{i} | p_i -\hat{p}_i| \\
& \leq \sum_{i \in A_1} | p_i -\hat{p}_i|
+ \sum_{i \in A_2} | p_i -\hat{p}_i| + 
\sum_{i \in A_3} | p_i -\hat{p}_i| \\
& \leq \sum_{i \in A_1} \alpha'' \sqrt{\frac{p_i}{m}} +  \sum_{i \in A_2}
\frac{5\alpha''}{m} + \frac{\alpha'}{m} + \frac{\gamma}{m} +  \sum_{i \in A_3}
 \alpha'\sqrt{\frac{p}{m}} + \frac{\alpha'^2}{m} + \frac{\gamma}{m} \\
 & \leq 28 \alpha'' + \alpha'\left(\sqrt{\frac{k}{m}} +1 \right) \frac{k\alpha'^2}{m} + \frac{k\gamma}{m} + \\
 & \leq \frac{\alpha}{8} + \frac{\alpha}{8} + \frac{\alpha}{8} + \frac{\alpha}{8} \\
 & \leq \frac{\alpha}{2}.
\end{align*}
Hence, by the union bound, with probability with $ 1- 20 k \beta$, 
\[
\ell_1(p, \hat{p}) \leq \frac{\alpha}{2}.
\]
Therefore in expectation,
\[
\EE[\ell_1(p, \hat{p}) ]
\leq \frac{\alpha}{2}
+ 20 k \beta \leq \alpha.
\]

\section{Extensions}
\label{app:extensions}
In this section, we modify our algorithms for the scenario
when users have different number of samples. Let $m_{\max}$ be a known upper bound on the number of samples a user has. For a value $m$, let $s_m$ be the number of users such that $m_u \geq m$. Let $\bar{m}$ be the median values of $m_u$.  We first state the main result, an analog of Theorem \ref{thm:kall}.

\begin{theorem}
Let $\epsilon \leq 1$. There exists a polynomial time algorithm $(\epsilon, \delta)$-differentially private algorithm $A$ such that
 \begin{equation}
 \label{eq:kall_var}
S^A_{m, \alpha, \epsilon, \delta} =
\cO \left(
\log^2 \frac{km_{\max}}{\alpha} \cdot 
\max \left( 
\frac{k}{\bar{m}\alpha^2} + \frac{k}{\sqrt{\bar{m}}\alpha\epsilon} \sqrt{\log \frac{1}{\delta}}, \frac{\sqrt{k}}{\epsilon}  \sqrt{\log \frac{1}{\delta}} \right) \right).
\end{equation}
\end{theorem}

First we use $\epsilon/2$ privacy budget find $\hat{m}$, a private estimate of $\bar{m}$, and $\hat{s}$, an estimate of $s_{\hat{m}}$ (the quantile of $\hat{m}$). We only keep the users with at least $\hat{m}$ samples, and select $\hat{m}$ samples from each of them. Hence we reduce the problem to the case when users have the same number of samples. Then we modify the algorithms for both the dense and sparse regimes so that they are differentially private even if the number of samples of a particular user changes. We use the remaining privacy budget for the modified algorithms. The privacy guarantee follows by the composition theorem. 

We first provide the algorithm for privately estimating $\bar{m}$ and the quantile of estimated $\bar{m}$, which serves as a stepping stone for extending our algorithms to variable number of samples per user.

\begin{lemma}
\label{lem:median}
Let $s \geq \frac{16\log^2 m_{\max}/{\beta}}{\epsilon}$.
There exists a polynomial time $(\epsilon, 0)$-algorithm that returns $\hat{m}$ and $\hat{s}$ such that with probability $\geq 1 - \beta$, the following holds,
\begin{equation}
|\hat{s} - s_{\hat{m}}  | \leq \frac{2\log^2 m_{\max}/\beta}{\epsilon},\;
\hat{m} \geq \frac{\bar{m}}{2},\;
s_{\hat{m}} \geq \frac{s}{4},\;
\hat{s} \geq \frac{3s}{8}.
\label{eq:med_prop}
\end{equation}

\end{lemma}
\begin{proof}
Divide $\{0,1,2,\ldots, m_{\max}\}$ to bins $b_i$ such that $b_0= 0$, $b_1 = 1$ and $b_i = 2*b_{i-1}$ for $i \geq 1$. There are $v = \log m_{\max}$ buckets. 

For any two adjacent datasets, $[t_{0}, t_{1}, t_{2}, \ldots, t_{v}]$ differ by two. Hence, we can add Laplace noise with parameter $\eta = 2/ \epsilon$ to each of them to obtain DP estimates. Let this be  $[t'_{0}, t'_{1}, \ldots, t'_{v}]$.

By the tail bounds of Laplace distribution and the union bound, for each $i$ with probability $1-\beta$,
\[
|t_i - t'_i| \leq \eta \log \frac{v}{\beta}.
\]
Furthermore, for any cumulative sets,
\[
\left|\sum_{i \geq j} t_i -\sum_{i \geq j} t'_i\right| \leq\sum_{i \geq j}\left| t_i - t'_i\right| \le \eta v \log \frac{v}{\beta}.
\]
Let $j^*$ be the largest $j$ such that
\[
\sum_{i \geq j} t'_i \geq \frac{s}{2} - \eta v \log \frac{v}{\beta} .
\]
The algorithms return $\hat{s} = \sum_{i \geq j^*} t'_i$ and $\hat{m} = b_{j^*}$. Then by the assumption on $s$:
\[
\hat{s}\ge \frac{s}{2}-\frac{2}{\epsilon}\log m_{\max}\log\frac{\log m_{\max}}{\beta}\ge \frac{s}{2}-\frac{s}{8}=\frac{3s}{8}.
\]

By the above cumulative equation sum,
\[
| \hat{s} - s_{\hat{m}}| = \left|\sum_{i \geq j^*} t_i - \sum_{i \geq j^*} t'_i\right| \leq \eta v\log \frac{v}{\beta}.
\]
\[
s_{\hat{m}}  = \sum_{i \geq j^*} t_i = 
\sum_{i \geq j^*} t'_i - \sum_{i \geq j^*} (t'_i - t_i)
\geq \frac{s}{2} - \eta v \log \frac{v}{\beta}  - \eta v \log \frac{v}{\beta} \geq \frac{s}{4}.
\]
Note that by definition of $j^*$, $\sum_{i \geq j^*+1} t'_i <s/2-v\log (v/\beta)$, and that $b_{j^*+1}=2b_{j^*}=2\hat{m}$, thus:
\[
s_{2\hat{m}}  = \sum_{i \geq j^*+1} t_i = 
\sum_{i \geq j^*+1} t'_i - \sum_{i \geq j^*+1} (t'_i - t_i)
\leq \frac{s}{2} -  \eta v \log \frac{v}{\beta} +  \eta v \log \frac{v}{\beta}  = \frac{s}{2}.
\]
Hence $2\hat{m} \geq \bar{m}$. This completes the proof.
\end{proof}

We proceed to discuss the algorithms for dense and sparse regimes. After we obtain $\hat{s}, \hat{m}$ from Lemma \ref{lem:median}, we choose the algorithm depending on the relation between $k$ and $\hat{m}$: if $k\le \hat{m}$, we use the algorithm for the dense regime; otherwise we use the one for the sparse regime. 

\subsection{Dense regime}
We first modify the hypothesis selection algorithm in \cite{bun2019private}. We cannot apply it directly because to ensure privacy, we cannot use the true number of users $s_{\hat{m}}$ and need to replace it with its private estimate $\hat{s}$. Hence we prove the following lemma to cope with this situation.
\begin{lemma}
\label{lem:phs_var}
Let $\hat{s}, \hat{m}$ satisfy ~\eqref{eq:med_prop} with $\epsilon = \epsilon'$.  
Given $d$ distributions $p_1, p_2, \ldots, p_d$ and $s$ independent samples from an unknown distribution $p$, such that $\min_{i} \ell_1(p_i, p) \leq \alpha$, there exists an $(\epsilon, 0)$-DP polynomial time algorithm that returns a distribution $p_i$ such that 
$
\ell_1(p_i, p) \leq 4 \alpha,
$
with probability $\geq 1- \beta$, if the number of samples satisfies,
\[
s \geq \frac{128\log^2(m_{\max}/\beta)}{3\alpha\epsilon'} + \frac{32 \log(4d/\beta)}{\alpha^2} +  \frac{64 \log(4d/\beta)}{3\alpha\epsilon}.
\]
\end{lemma}
\begin{proof}

Let $H$ and $H'$ be two distributions over the domain $\cX$ and define the Scheffe set
\[
\cW_1 = \{x\in\cX: H(x)>H'(x) \}.
\]
Define $p_1=H(\cW_1), p_2=H'(\cW_1)$, for some distribution $P$ define $\tau=P(\cW_1)$. Note that $p_1>p_2$ and $p_1-p_2=d_{TV}(H, H')$. 

Let $D$ be a dataset of size $s_{\hat{m}}$ drawn i.i.d. from $P$. Define the following quantities which serve as empirical estimates of $P(\cW_1)$,
\[
\hat{P}(\cW_1):=\hat{\tau}:=\frac{1}{\hat{s}}|\{x\in D: x\in\cW_1\}|,\quad P_{\hat{m}}(\cW_1):=\tau_{\hat{m}}:=\frac{1}{s_{\hat{m}}}|\{x\in D: x\in\cW_1\}|.
\]

Let $\zeta>0$ be the approximation parameter. Consider the function
\[
\hat{\Gamma}_\zeta(H, H', D)=\begin{cases}\hat{s} & p_1-p_2\le (2+\zeta)\alpha;\\
\hat{s}\cdot \max\{0, \hat{\tau}-(p_2+(1+\zeta/2)\alpha)\} &\text{otherwise.}
\end{cases}
\]

According to \cite[Lemma 3.1, Lemma 3.3]{bun2019private}, $\hat{\Gamma}_{\zeta}$ has the following properties,
\begin{lemma}[\cite{bun2019private}, Lemma 3.1]
If $d_{TV}(P, H)\le \alpha$ and $|\hat{\tau}-\tau|<\zeta\alpha/4$, then $\hat{\Gamma}_{\zeta}(H, H', D)>\zeta\alpha\hat{s}/4$.
\label{lem:bun3.1}
\end{lemma}

\begin{lemma}[\cite{bun2019private}, Lemma 3.3]
If $d_{TV}(P, H')\le \alpha$ , $|\hat{\tau}-\tau|<\zeta\alpha/4$, and $\hat{\Gamma}_{\zeta}(H, H', D)>0$, then $d_{TV}(H, H')\le (2+\zeta)\alpha$.
\label{lem:bun3.3}
\end{lemma}

Define the score functions for each $H_j\in\H$
\[
\hat{S}(H_j, D)=\min_{H_k\in\H}\hat{\Gamma}_\zeta(H_j, H_k, D).
\]
Output a random hypothesis $\hat{H}$ according to the distribution
\[
\Pr[\hat{H}=H_j]\propto\exp\left(\frac{\hat{S}(H_j, D)}{2\epsilon}\right).
\]

First note that if $d_{TV}(P, H)<\alpha$, then using Hoeffding's inequality, we have  with probability at least $1-2\exp(-s_{\hat{m}}\zeta^2\alpha^2/32)$, 
\[
|\tau_{\hat{m}}-\tau|<\zeta\alpha/8.
\]

Assume that there exists $H^*\in\H$ such that $d_{TV}(P, H^*)\le \alpha$. Define $\cW_j=\{x\in\cX: H^*(x)>H_j(x)\}$. Conditioned on that the inequalities in Lemma \ref{lem:median} hold, by the union bound, with probability at least $1-2d\exp(-s_{\hat{m}}\zeta^2\alpha^2/8)\ge 1-2d\exp(-s\zeta^2\alpha^2/32)$ over the draws of $D$, for all $j$ we have 
\[
|P(\cW_j)-P_{\hat{m}}(\cW_j)|\le \zeta\alpha/8.
\]

Due to the inequalities in Lemma \ref{lem:median}, the following holds uniformly for all $j$,
\[
|\hat{P}(\cW_j)-P_{\hat{m}}(\cW_j)|\le \left|\frac{1}{\hat{s}}-\frac{1}{s_{\hat{m}}}\right|s_{\hat{m}}=\frac{|\hat{s}-s_{\hat{m}}|}{\hat{s}}\le \frac{16\log^2(m_{\max}/\beta)}{3s\epsilon'}
\]
Hence as long as $s>\frac{128\log^2(m_{\max}/\beta)}{3\zeta\alpha\epsilon'}$, the above quantity is bounded by $\zeta\alpha/8$. We have
\[
|P(\cW_j)-\hat{P}(\cW_j)|\le  |P(\cW_j)-P_{\hat{m}}(\cW_j)|+|\hat{P}(\cW_j)-P_{\hat{m}}(\cW_j)|\le \frac{\zeta\alpha}{4}.
\]

By Lemma \ref{lem:bun3.1} we have $\hat{\Gamma}_\zeta(H^*, H_j, D)>\zeta\alpha \hat{s}/4\ge 3\zeta\alpha s/32$. This implies $\hat{S}(H^*, D)>  3\zeta\alpha s/32$.

By the utility of the exponential mechanism, with probability at least $1-\beta/2$, the output hypothesis $\hat H$ satisfies
\begin{align*}
    \hat{S}(\hat{H}, D)&\ge\hat{S}(H^*, D)-\frac{2\log(2d/\beta)}{\epsilon}\\
    &\ge \frac{3\zeta\alpha s}{32}-\frac{2\log(2d/\beta)}{\epsilon}.
\end{align*}
As long as $s\ge \frac{32\log(4d/\beta)}{\zeta^2\alpha^2}+\frac{64\log(2d/\beta)}{3\zeta\alpha\epsilon}$, together with probability at least $1-\beta$ , $\hat{S}(\hat{H}, D)>0$, which implies that $\hat{\Gamma}_\zeta(\hat{H}, H^*, D)>0$. Since in addition $d_{TV}(P, H^*)\le \alpha$, we have $d_{TV}(\hat H, H^*)\le (2+\zeta)\alpha$ by Lemma \ref{lem:bun3.3} and hence $d_{TV}(\hat{H}, P)\le  (3+\zeta)\alpha$. Setting $\zeta = 1$ gives the desired result.
\end{proof}

\begin{theorem}
\label{thm:binom_var}
Suppose there are $s$ users such that user $u$ has $m_u$ i.i.d. samples from $Ber(p)$. Let $\hat{s}, \hat{m}$ satisfy ~\eqref{eq:med_prop} with $\epsilon = \epsilon'$. Let $
s \geq \frac{128\log^2(m_{\max}/\beta)}{3\alpha\epsilon'} + \frac{32 \log(80m_{\max}/\alpha\beta)}{\alpha^2} +  \frac{64 \log(80m_{\max}/\alpha\beta)}{3\alpha\epsilon}
$.
There exists a polynomial time $(\epsilon, 0)$ differentially private algorithm that returns $\hat{p}$ such that with probability at least $ 1- \beta$,
\[
|p - \hat{p}| \leq \frac{4}{5}\alpha\max \left( \frac{1}{\hat{m}} , \frac{\sqrt{p(1-p)}}{\sqrt{\hat{m}}} \right).
\]
\end{theorem}
\begin{proof}
We sample $\hat{m}$ samples from all users that have least $\hat{m}$ samples. Hence we obtain $s_{\hat{m}}$ i.i.d samples from $\Bin(\hat{m}, p)$. Let $c$ be the constant in Theorem \ref{thm:binomial}. We then apply the modified hypothesis selection algorithm in Lemma \ref{lem:phs_var} with the hypothesis class  $\cQ=\{\Bin(\hat{m}, p), p\in\cP\}$ where $\cP=\{0, \frac{c\alpha}{20\hat{m}}, \frac{2c\alpha}{20\hat{m}}..., 1\}$. The total number of hypotheses is $d=\frac{20\hat{m}}{c\alpha}$. The sample complexity comes from Lemma \ref{lem:phs_var} and utility follows by the argument in Theorem \ref{thm:main_upper} with $m$ replaced by $\hat{m}$. 

By Theorem~\ref{thm:binomial}, for every $\Bin(\hat{m},p)$ there exists a $\Bin(\hat{m},p')$ in $\cQ$ such that 
\[
\ell_1(\Bin(\hat{m},p), \Bin(\hat{m},p')
\leq  \frac{c\alpha}{5}.
\]
Hence, by Lemma~\ref{lem:phs_var}, if 
\[
s = \Omega \left( \frac{128\log^2(m_{\max}/\beta)}{3\alpha\epsilon'} + \frac{32 \log(80m_{\max}/\alpha\beta)}{\alpha^2} +  \frac{64 \log(80m_{\max}/\alpha\beta)}{3\alpha\epsilon} \right),
\]
there is an algorithm that returns a distribution $\Bin(m,\hat{p}) \in \cQ$ such that 
\[
\ell_1(\Bin(\hat{m},p), \Bin(\hat{m},\hat{p})
\leq  \frac{4c\alpha}{5},
\]
with probability $\geq 1 - \beta$.
Therefore, by the lower bound in Theorem~\ref{thm:binomial}, the resulting $\hat{p}$ satisfies
\[
\min \left(\hat{m} | p - \hat{p}|,  \frac{\sqrt{\hat{m}} |p - \hat{p}|}{ \sqrt{p(1-p)}}, 1\right) \leq \frac{4\alpha}{5},
\]
with probability $\geq 1- \beta$. Since $\frac{4\alpha}{5} \leq 1$ and $\hat{m}\ge \bar{m}/2$, this implies that with probability   $\geq 1- \beta$,
\[
|p - \hat{p}| \leq \frac{4\alpha}{5} \max \left( \frac{1}{\hat{m}} , \frac{\sqrt{p(1-p)}}{\sqrt{\hat{m}}} \right)\le \frac{4\alpha}{5} \max \left( \frac{2}{\bar{m}} , \frac{\sqrt{2p(1-p)}}{\sqrt{\bar{m}}} \right).
\]
The expectation bound follows by setting $\beta = \alpha/5m_{\max}$,
\[
\EE[|p - \hat{p}|]
\leq  \frac{4\alpha}{5} \max \left( \frac{2}{\bar{m}} , \frac{\sqrt{2p(1-p)}}{\sqrt{\bar{m}}} \right) + \frac{\alpha}{5m_{\max}}
\leq\alpha \max \left( \frac{2}{\bar{m}} , \frac{\sqrt{2p(1-p)}}{\sqrt{\bar{m}}} \right).
\]
\end{proof}
\begin{theorem}[Dense regime]
\label{thm:ksmall_var}
Let $k \leq \hat{m}$ and $\epsilon \leq 1$. There exists a polynomial time  $(\epsilon, \delta)$-differentially private algorithm $A$ such that
 \[
S^A_{m, \alpha, \epsilon, \delta} =
\cO \left(
\log^2 \frac{km_{\max}}{\alpha}\cdot 
\max \left( 
\frac{k}{\bar{m}\alpha^2} + \frac{k}{\sqrt{\bar{m}}\alpha\epsilon}\sqrt{\log \frac{1}{\delta}}, \frac{\sqrt{k}}{\epsilon}\sqrt{\log \frac{1}{\delta}}  \right) \right).
\]
\end{theorem}
\begin{proof}
Let $\beta>0$ be the probability guarantee to be chosen later. Use $\epsilon_1=\epsilon/2$ budget to obtain $\hat{s}, \hat{m}$ using Lemma \ref{lem:median}, which satisfy ~\eqref{eq:med_prop}
with probability at least $1-\beta$ as long as $s\ge \frac{16\log^2 m_{\max}/{\beta}}{\epsilon/2}$. 

Define $\epsilon_2=\frac{\epsilon}{8\sqrt{(k+1)\log(2/\delta)}}, \alpha'=\min \left( \frac{\sqrt{\hat{m}}\alpha}{2\sqrt{k}}, 1\right)$. Under the condition above, by union bound and applying Theorem \ref{thm:binom_var} with $\epsilon'=\epsilon_1, \epsilon=\epsilon_2, \alpha=\alpha'$, with probability at least $1-k\beta$, for all $\hat{p}_i$ we have
\[
|p_i - \hat{p}_i| \leq \frac{4}{5}\alpha'\max \left( \frac{1}{\hat{m}} , \frac{\sqrt{p(1-p)}}{\sqrt{\hat{m}}}\right),
\]
as long as 
\begin{align}
    \label{equ:comp_dense_var}
    s&=\Omega\left(\log^2\frac{m_{\max}}{\alpha\beta}\max\left(\frac{k}{\bar{m}\alpha^2}+\frac{\sqrt{k}}{\alpha\epsilon_2\sqrt{\bar{m}}}, \frac{1}{\epsilon_2} \right)\right)\\
    &\ge \frac{128\log^2(m_{\max}/\beta)}{3\alpha'\epsilon_1} + \frac{32 \log(80m_{\max}/\alpha'\beta)}{(\alpha')^2} +  \frac{64 \log(80m_{\max}/\alpha'\beta)}{3\alpha'\epsilon_2}\nonumber.
\end{align}

Note that this satisfies the condition on $s$ in Lemma \ref{lem:median}. Together with probability at least $1-(k+1)\beta$:
\begin{align*}
\ell_1(p,\hat{p})
& = \sum_i |p_i-\hat{p}_i| \\ 
& \leq  \frac{4}{5}\alpha' \sum_i \max \left( \frac{1}{\hat{m}} , \frac{\sqrt{p_i(1-p_i)}}{\sqrt{\hat{m}}} \right) \\
& \leq  \frac{4}{5}\alpha' \sum_i  \frac{1}{\hat{m}} +  \frac{\sqrt{p_i}}{\sqrt{\hat{m}}}  \\
& \leq    \frac{4}{5}\left(\frac{\alpha' k}{\hat{m}} + 
\frac{\alpha' \sqrt{k}}{\sqrt{\hat{m}}}\right) \\
& \leq    2\frac{4}{5}
\frac{\alpha' \sqrt{k}}{\sqrt{\hat{m}}} \\
& \leq \frac{4}{5}\alpha,
\end{align*}

Choosing $\beta = \frac{\alpha}{40k}$,
\[
\EE[\ell_1(\hat{p}, p)]\le \frac{4\alpha}{5}+ 2(k+1)\beta=\alpha.
\]
Plug in $\epsilon_2$ and $\beta$ in \eqref{equ:comp_dense_var} we obtain the desired user complexity. Privacy guarantee follows by  the composition theorem.
\end{proof}
\subsection{Sparse regime}
\begin{lemma}
\label{lem:small_var}
Let $\hat{s}, \hat{m}$ satisfy ~\eqref{eq:med_prop} with $\epsilon = \epsilon'$. 
Let $p \leq \min(c/\hat{m}, 1/2)$.  
Let $s \geq 64 e^{3c} \max(c, 1) \log \frac{3}{\beta}$ and $ s \geq \frac{128e^{3c}}{\alpha^2} \log \frac{3}{\beta} +  \frac{32e^{3c}}{\gamma\epsilon'} \log^2 \frac{3m_{\max}}{\beta}+\frac{16e^{3c}}{\gamma\epsilon}\log\frac{3}{\beta}$.
There exists a polynomial time $(\epsilon,\delta)$-estimator $\hat{p}$ such that with probability at least $ 1- \beta$,
\[
|p - \hat{p}|  \leq   \sqrt{\frac{p\alpha^2}{\hat{m}}} + \frac{\alpha^2}{\hat{m}} + \frac{\gamma}{\hat{m}\epsilon}.
\]
\end{lemma}
\begin{proof}

We modify the algorithm for the sparse regime as follows.

Let $\cU_{\hat{m}}$ be the users who have at least $\hat{m}$ samples.
Similar to ~\eqref{eq:sparse-p_hat}, we find $\hat{p}$ such that,
\[
    (1-\hat{p})^{\hat{m}} = \max \left( \min \left(\frac{1}{\hat{s}} \sum_{u\in\cU_{\hat{m}}} 1_{N(u) = 0} + \frac{Z}{\hat{s}}, 1 \right), 0 \right),
\]
where $Z=Lap(1/\epsilon)$. Therefore,
\begin{align*}
    |(1-\hat{p})^{\hat{m}} - (1-{p})^{\hat{m}}| &\leq 
    \left \lvert \frac{1}{\hat{s}} \sum_{u\in\cU_{\hat{m}}} 1_{N(u) = 0} + \frac{Z}{\hat{s}} - (1-p)^{\hat{m}} \right \rvert\\
    &\le \left|\frac{1}{\hat{s}} \sum_{u\in\cU_{\hat{m}}} 1_{N(u) = 0}-\frac{1}{s_{\hat{m}}} \sum_{u\in\cU_{\hat{m}}} 1_{N(u) = 0} \right|+ \frac{|Z|}{\hat{s}}\\
    &+ \left|\frac{1}{s_{\hat{m}}} \sum_{u\in\cU_{\hat{m}}} 1_{N(u) = 0}-(1-p)^{\hat{m}}\right|.
\end{align*}
From Lemma \ref{lem:median}, with probability at least $1-\beta$, the first term is upper bounded by
\[
\left|\frac{1}{\hat{s}}-\frac{1}{s_{\hat{m}}}\right|\left|\sum_{u\in\cU_{\hat{m}}} 1_{N(u) = 0}\right|\le \left|\frac{\hat{s}-s_{\hat{m}}}{\hat{s}}\right|\le \frac{16\log^2(m_{\max}/\beta)}{3s\epsilon'}.
\]

The second and third term are bounded similar to Lemma \ref{lem:small} using Laplace tail bounds and Bernstein's inequality. With probability $1-4\beta$,
\[
\frac{|Z|}{\hat{s}}+ \left|\frac{1}{s_{\hat{m}}} \sum_{u\in\cU_{\hat{m}}} 1_{N(u) = 0}-(1-p)^{\hat{m}}\right|\le \frac{\log(1/\beta)}{\hat{s}\epsilon}+ 4\sqrt{\frac{\hat{m}p \log  \frac{1}{\beta}}{s_{\hat{m}}}}
 + 4 \frac{\log \frac{1}{\beta}}{s_{\hat{m}}}.
\]

Together with probability at least $1-5\beta$,
\begin{align*}
    e^{-1.5c}\min\{\hat{m}|\hat{p}-p|, 0.5\}&\le |(1-\hat{p})^{\hat{m}} - (1-{p})^{\hat{m}}|\\
    & \le \frac{16\log^2(m_{\max}/\beta)}{3s\epsilon'}+\frac{\log(1/\beta)}{\hat{s}\epsilon}+ 4\sqrt{\frac{\hat{m}p \log  \frac{1}{\beta}}{s_{\hat{m}}}}
    + 4 \frac{\log \frac{1}{\beta}}{s_{\hat{m}}} \\
    &\le\frac{16\log^2(m_{\max}/\beta)}{3s\epsilon'}+\frac{8\log(1/\beta)}{3s\epsilon}+ 8\sqrt{\frac{\hat{m}p \log  \frac{1}{\beta}}{s}}
    +  \frac{16\log \frac{1}{\beta}}{s}.
\end{align*}
The last inequality is due to $\hat{s}\ge 3s/8$ and $s_{\hat{m}}\ge s/4$. 

If $s\ge 256e^{3c}p\hat{m} \log(3/\beta)$, then the right hand side is upper bounded by $e^{-1.5c}/2$. Thus,
\[
e^{-1.5c}\hat{m}|\hat{p}-p|\le \frac{16\log^2(m_{\max}/\beta)}{3s\epsilon'}+\frac{8\log(1/\beta)}{3s\epsilon}+8\sqrt{\frac{\hat{m}p \log  \frac{1}{\beta}}{s}}
    +  \frac{16\log \frac{1}{\beta}}{s}.
\]
If $ s \geq \frac{128e^{3c}}{\alpha^2} \log \frac{3}{\beta} +  \frac{32e^{3c}}{\gamma\epsilon'} \log^2 \frac{3m_{\max}}{\beta}+\frac{16e^{3c}}{\gamma\epsilon}\log\frac{3}{\beta}$,
\[
|\hat{p}-p|\le \sqrt{\frac{p\alpha^2}{\hat{m}}} + \frac{\alpha^2}{\hat{m}} + \frac{\gamma}{\hat{m}}.
\]

In the end we get a result similar to Lemma 3.
\end{proof}

\begin{theorem}
\label{thm:klarge_var}
Let $\epsilon \leq 1$ and $k \geq \hat{m}$. There exists a polynomial time  $(\epsilon, \delta)$-differentially private algorithm $A$ such that
 \[
S^A_{m, \alpha, \epsilon, \delta} 
= \cO \left( \log^2 \frac{km_{\max}}{\alpha}  \cdot \left(
\frac{k}{\bar{m}\alpha^2} + \frac{k}{\sqrt{\bar{m}}\epsilon\alpha} \sqrt{\log \frac{1}{\delta}} \right) \right).
\]
\end{theorem}
\begin{proof}
Like the algorithm for the dense regime, we first use $\epsilon_1=\frac{\epsilon}{2}$ budget to estimate $\hat{s}, \hat{m}$. Then we define the following parameters,
\[
\epsilon_2 = \frac{\epsilon/2}{8\sqrt{\min(k,\hat{m}) \log \frac{1}{\delta/2}}},\;\beta = \frac{\alpha}{40k},\; \alpha' =  \min \left( \frac{\sqrt{\hat{m}}\alpha}{8\sqrt{k}}, 1\right),\; \alpha'' = \frac{\alpha}{240},\;\gamma = \frac{\hat{m} \alpha}{8k}
\]
The proof follows similarly as Theorem \ref{thm:klarge}. 

\textbf{Algorithm}:
For every symbol we first calculate the 
probability using the algorithm in Theorem~\ref{thm:binom_var} with $\epsilon'=\epsilon_1, \epsilon = \epsilon_2$, $\alpha = \alpha''$ and error probability $\beta$. If the estimated probability is less than $2/\hat{m}$, we use the algorithm from Lemma~\ref{lem:small_var} with $\epsilon'=\epsilon_1, \epsilon = \epsilon_2$, $\alpha = \alpha'$, $\gamma = \gamma$, and error probability $\beta$.
Let ${p}'$ be the output of the first step and the $p''$ be the output of Lemma~\ref{lem:small_var}.
The error of the algorithm is 
\[
|p - \hat{p}| = | p - {p}'| 1_{{p}' > 2/m} + | p - {p}''|
1_{{p}' \leq 2/m}.
\]

\textbf{Sample complexity}: The sample complexity would be the sum of sample complexities of Theorem ~\ref{thm:binom_var} and Lemma~\ref{lem:small_var} with appropriate parameters.
Hence,
\begin{align*}
s &\geq \frac{128\log^2(m_{\max}/\beta)}{3\alpha''\epsilon_1} + \frac{32 \log(80m_{\max}/c\alpha''\beta)}{\alpha''^2} +  \frac{64 \log(80m_{\max}/c\alpha''\beta)}{3\alpha''\epsilon_2}\\
&\quad + \frac{128e^{3c}}{\alpha'^2} \log \frac{3}{\beta} +  \frac{32e^{3c}}{\gamma\epsilon_1} \log^2 \frac{3m_{\max}}{\beta}+\frac{16e^{3c}}{\gamma\epsilon_2}\log\frac{3}{\beta}.
\end{align*}
Hence, for a sufficient large constant $b$, if 
\begin{align*}
s \geq b \log^2 \frac{km_{\max}}{\alpha} \cdot \left( 
\frac{k}{\bar{m}\alpha^2} + \frac{k}{\sqrt{\bar{m}}\epsilon\alpha}  \sqrt{\log \frac{1}{\delta}}
\right).
\end{align*}
Note that since $k \geq \hat{m}$, the above bound implies that $s \geq b \sqrt{\hat{m}}$, hence the bound also satisfies conditions in Lemma~\ref{lem:small_var} and Lemma~\ref{lem:median}. 

Following the same argument as Theorem \ref{thm:klarge}, the algorithm after we obtain $\hat{s}, \hat{m}$ is $(\epsilon/2, \delta/2)$ private. Using the naive composition theorem, the entire algorithm is $(\epsilon, \delta)$ private. 

Utility follows by the argument in Theorem \ref{thm:klarge} with $m$ replaced by $\hat{m}$.
\end{proof}

\end{document}